\documentclass{article}

\usepackage{microtype}
\usepackage{graphicx}
\usepackage{subcaption}
\usepackage{booktabs} 

\usepackage{hyperref}

\usepackage{bbm}
\usepackage{graphicx}
\usepackage{amsmath}
\usepackage{amsthm}
\usepackage{amssymb}
\usepackage{verbatim}
\usepackage{floatpag}
\usepackage{subeqnarray}
\usepackage{mathrsfs}
\usepackage{cancel}
\usepackage{bm}
\usepackage{color,soul}
\usepackage[capitalize]{cleveref}
\usepackage[round]{natbib}
\usepackage{url}
\usepackage[ruled]{algorithm2e}

\newtheorem*{theorem*}{Theorem}
\newtheorem{theorem}{Theorem}
\newtheorem{lemma}{Lemma}
\newtheorem{proposition}{Proposition}

\newcommand{\E}{\mathbb{E}}
\newcommand{\euclid}{\mathbb{R}}
\newcommand{\A}{\mathcal{A}}
\newcommand{\steinset}{\mathcal{C}_k}
\newcommand{\rkhs}{\mathcal{K}_{k}}
\newcommand{\kccsd}{\mathcal{S}(q, \A_p, \steinset)}
\newcommand{\norm}[1]{\left\lVert#1\right\rVert}
\newcommand{\prob}{\mathbb{P}}
\newcommand{\bmx}{\bm{x}}
\newcommand{\bmy}{\bm{y}}

\usepackage[preprint,nonatbib]{neurips_2020}

\usepackage{authblk}
\title{Kernelized Complete Conditional Stein Discrepancies}

\author[1]{\textbf{Raghav Singhal}}
\author[2]{\textbf{Xintian Han}}
\author[2]{\textbf{{Saad Lahlou}}}
\author[1,2]{\textbf{{Rajesh Ranganath}}}
\affil[1]{Courant Institute of Mathematical Sciences, New York University}
\affil[2]{Center for Data Science, New York University}

\begin{document}
\maketitle

\begin{abstract}
  Much of machine learning relies on comparing distributions with
  discrepancy measures. Stein's method creates discrepancy
  measures between two distributions that require only the
  unnormalized density of one and samples from the
  other. Stein discrepancies can be combined with kernels to
  define kernelized Stein discrepancies (\textsc{ksd}s).
  While kernels make Stein discrepancies tractable, they
  pose several challenges in high dimensions. We introduce
  kernelized complete conditional Stein discrepancies (\textsc{kcc-sd}s).
  Complete conditionals turn a multivariate
  distribution into multiple univariate distributions.
  We show that \textsc{kcc-sd}s
  distinguish distributions. To show the efficacy of \textsc{kcc-sd}s in
  distinguishing distributions, we introduce a
  goodness-of-fit test using \textsc{kcc-sd}s. We empirically show that
  \textsc{kcc-sd}s have higher power over baselines and use
  \textsc{kcc-sd}s to assess sample quality in Markov chain Monte Carlo.
\end{abstract}

\section{Introduction}\label{sec:intro}
Discrepancy measures that compare a distribution $p$,
known up to normalization, with a distribution $q$, known
via samples from it,
can be used for finding good variational approximations \citep{ranganath2016operator, liu2016stein},
checking the quality of \textsc{mcmc}
samplers \citep{gorham2015measuring, gorham2017measuring}, goodness-of-fit testing \citep{liu2016kernelized}, parameter estimation \citep{barp2019kernel} and multiple model comparison \citep{lim2019kernel}.
There are several difficulties with using traditional discrepancies
like Wasserstein metrics or total variation distance for these tasks.
Mainly, $p$ can be hard to sample so expectations under $p$ cannot
be computed. These challenges lead to the following desiderata for a
discrepancy $D$ \citep{gorham2015measuring}.
\begin{enumerate}
\item\label{desiderata_1}
  \textbf{Tractable} $D$ uses samples
  from $q$, and evaluations of (unnormalized) $p$.
\item\label{desiderata_2}
  \textbf{Distinguishing Distributions} $D(p, q) = 0$ if and only if $p$ is
  equal in distribution to $q$.
\end{enumerate}
These desiderata ensure that the discrepancy is non zero when $p$
does not equal $q$ and that it can be easily computed. To meet these desiderata,
\citet{chwialkowski2016kernel, oates2017control, gorham2017measuring, liu2016kernelized}
developed kernelized Stein discrepancies (\textsc{ksd}s).
\textsc{ksd}s measure the expectation of functions under $q$
that have expectation zero
under $p$. These functions are constructed by applying Stein's operator to a reproducing
kernel Hilbert space (\textsc{rkhs}).

In high dimensions, many popular kernels evaluated on a pair of points
are near zero. Thus, \textsc{ksd}s in high dimensions can
be near zero, making detecting differences between high dimensional
distributions difficult. The median heuristic can be used to address this to
some extent, but \textsc{ksd}s with the median heuristic can still have low
power in moderately high dimensions (see Figure 1, \citet{jitkrittum2017linear}).
We develop kernelized complete
conditional Stein discrepancies (\textsc{kcc-sd}s). These discrepancies use
complete conditionals: the distribution of one variable
given the rest. Complete conditionals are univariate distributions.
Rather than using multivariate kernels, \textsc{kcc-sd}s use
univariate kernels to ensure the complete conditionals match,
making it easier to compare distributions in high dimensions.

A given Stein discrepancy relies on a supremum over a class of
test functions called the Stein set. \textsc{kcc-sd}s differ from
\textsc{ksd}s in that \textsc{kcc-sd}s compute a separate supremum for each
complete conditional. An immediate question is whether there is a computable
closed form and whether the discrepancy can be used to distinguish distributions.
We show that \textsc{kcc-sd}s have a closed form and distinguish between
distributions. Computing \textsc{kcc-sd} requires sampling from a complete
conditional of $q$, which can be infeasible in some instances. To address this,
we introduce approximate \textsc{kcc-sd} that uses a learned sampler for the complete
conditional.

To show the efficacy of \textsc{kcc-sd} and approximate \textsc{kcc-sd} in distinguishing
distributions we introduce a goodness-of-fit test
\citep{chwialkowski2016kernel}. We show that \textsc{kcc-sd} and
approximate \textsc{kcc-sd} have
higher power than \textsc{ksd} and other baselines.
We empirically show
that approximate \textsc{kcc-sd} does not suffer from a loss in power due to an
increase in dimension. We also demonstrate that \textsc{kcc-sd} and
approximate \textsc{kcc-sd} can be used to select sampler hyperparameters and
can be used to assess sample quality in a Gibbs sampler.

\paragraph{Related Work.}
There have been several lines of work which use factorizations of the
distribution $p$ to address the curse of dimensionality.
\citet{wang2017stein, zhuo2017message} use the Markov blanket of each node
to define a graphical version of \textsc{ksd} to alleviate the curse of
dimensionality. Our approach does not presume a graphical structure of $p$ or
$q$. \citet{wang2017stein} shows that unless the graphical structure for $p, q_n$
match, the graph based \textsc{ksd} converging to zero does not imply that
$q_n$ converges in distribution to $p$.

\citet{gong2020sliced} introduce the maximum
sliced kernelized Stein discrepancy (\textsc{MAXsksd}), which also uses
low-dimensional kernels by projecting into a $1$-dimensional
space. Computing \textsc{MAXsksd} requires
optimizing a projection direction that is specific to both
sampling distribution $q$ and the unnormalized distribution $p$.
This can be expensive when testing multiple distributions
or when changing the parameters of an
unnormalized model to fit a collection of
samples. Approximate \textsc{kcc-sd} requires learning conditional
distributions specific only to the sampling distribution $q$.
Similar to approximate \textsc{kcc-sd} with parametric
conditional estimates, the closed form for \textsc{MAXsksd} depends on
the optimal direction, therefore the power of their method depends on the
quality of the optimization, which can be difficult to guarantee for
arbitrary log probabilities.

\textsc{ksd}s suffer from a computational cost that is quadratic in the number
of samples. \citet{huggins2018random} develop random feature Stein
discrepancies \textsc{r}$\Phi$\textsc{sd}, which run in linear time
and perform as well as or better than quadratic-time \textsc{ksd}s; these
ideas can be applied to \textsc{kcc-sd}s. \citet{chen2018stein} introduces
the Stein points method which introduces a method to
select points to minimize the Stein discrepancy between the
empirical distribution supported at the selected points and the
posterior.

\citet{chwialkowski2016kernel} introduced \textsc{ksd} as a test statistic for
a goodness-of-fit test, which also suffers from the curse of dimensionality due
to the use of kernels in high dimensions, along with a computational cost
quadratic in the number of samples. \citet{jitkrittum2017linear} introduce a
linear-time discrepancy, finite-set Stein discrepancy (\textsc{fssd}).
The authors introduce an optimized version of \textsc{fssd} which allows one to
find features that best indicate the differences between the samples and the
target density. \textsc{fssd} while having a computational cost linear
in sample size, also leads to a test with lower power in high dimensions.


\section{Kernelized Stein Discrepancies}
Stein's method provides recipes for constructing expectation zero test functions
of distributions known up to a normalization constant. For a distribution $p$
with a integrable score function\footnote[1]{The score function in general is
the gradient of the log-likelihood with respect to the parameter vector. We
however refer to the gradient of the log-likelihood with respect to the
input \citep{hyvarinen2005estimation}.}, $\nabla_{\bm{x}} \log p(\bm{x})$,
we can create a \emph{Stein operator}, $\A_{p}$, that
acts on test functions $f: \euclid^d \rightarrow \euclid^d$ satisfying
regularity and boundary conditions (Proposition 1, \citep{gorham2015measuring}),
such that

\begin{align*}
  \E_{p(\bm{x})} \left[ \A_{p(\bm{x})} f(\bm{x}) \right] = 0 .
\end{align*}
This relation called \emph{Stein's identity} is used
to create \emph{Stein discrepancies} $\mathcal{S}(q, \A_{p}, \mathcal{H})$,
defined as
\begin{align*}
  \mathcal{S}(q, \A_{p}, \mathcal{H}) &= \sup_{f \in \mathcal{H}}
                                        \left| \E_{q(\bm{x})}[\A_{p(\bm{x})}f(\bm{x})] -
                                        \E_{p(\bm{x})}[\A_{p(\bm{x})}f(\bm{x})] \right| \\
   &= \sup_{f \in \mathcal{H}} \left|
  \E_{q(\bm{x})} \left[ \A_{p(\bm{x})} f(\bm{x}) \right] \right| \ ,
\end{align*}
where $\mathcal{H}$ is a function space known as the \emph{Stein set},
with its functions satisfying some boundary and regularity conditions. To make
the Stein discrepancy simpler to compute, \citet{chwialkowski2016kernel,
 oates2017control, gorham2017measuring, liu2016kernelized} used reproducing kernel
Hilbert spaces (\textsc{rkhs}) as the Stein set to
introduce kernelized Stein discrepancies (\textsc{ksd}). Let
$k: \euclid^{d} \times \euclid^{d} \rightarrow \euclid$ be the kernel of an
\textsc{rkhs} $\mathcal{K}_k$, the \textsc{rkhs} consists of functions, $g :
\euclid^d \rightarrow \euclid$, satisfying the reproducing property $g(\bm{x}) =
{\langle {g, k(\bm{x}, \cdot)} \rangle}_{\mathcal{K}_k}$. \textsc{ksd}s are defined
by the Stein set
\begin{align*}
  \mathcal{G}_k = \left \{g = (g_1, \dots, g_d): g_{i} \in
\mathcal{K}_k, \sum_{i=1}^{d} \norm{g_i}_{\mathcal{K}_k} \leq 1 \right \} \ .
\end{align*}
This construction of the Stein set using an \textsc{rkhs} ensures that the
Stein discrepancy has a closed form.
\begin{proposition}[Gorham and Mackey, 2017]
  Suppose $k \in C^{(1, 1)}$ and for each $j \in \{1, \dots, d\}$, define the
  Stein kernel as follows:
  \begin{align} \label{eq:ksd_closed_form}
    k_{0}^{j}(\bm{x}, \bm{y}) &= b_{j}(\bm{x}) b_j(\bm{y}) k(\bm{x}, \bm{y})
       + \nabla_{x_j} \nabla_{y_j} k(\bm{x}, \bm{y}) \\
    \nonumber
    & \quad +b_j(\bm{x}) \nabla_{y_j} k(\bm{x}, \bm{y}) +
       b_j(\bm{y}) \nabla_{x_j} k(\bm{x}, \bm{y}) \ ,
  \end{align}
  where $b_j(\bm{x}) = \nabla_{x_j} \log p(\bm{x})$.
  If $\sum_{j=1}^{d} \E_{q}[{k_{0}^{j}(\bm{x}, \bm{x})}^{1/2}]
  < \infty$, then \textsc{ksd} has a closed form. Given by
  $\mathcal{S}(q, \A_{p}, \mathcal{G}_k) = \norm{\bm{w}}_{2}$,
  where $ w_j^{2} \equiv
  \E_{q(\bm{x}) \times q(\bm{y})} \left[ k_0^{j}(\bm{x}, \bm{y}) \right]$
  with $\bm{x}, \bm{y} \overset{i.i.d}{\sim} q$.
\end{proposition}
When the distribution $p$ lies in the class of distantly
dissipative distributions \citep{eberle2016reflection}, \textsc{ksd}s provably detect
convergence and non-convergence for $d=1$. That is $\mathcal{S}(q_n, \A_p,
\mathcal{G}_k) \rightarrow 0$ if and only if $q_n \Rightarrow p$ for sequences
$\{q_n\}$, using kernels like the radial basis
function or the inverse multi-quadratic (\textsc{imq}), \citep{gorham2017measuring}.
In $d > 2$, the \textsc{ksd} with thin
tailed kernels like the \textsc{rbf} does not detect non-convergence. But the
\textsc{ksd} with the \textsc{imq}
kernel with $\beta \in (0, 1)$ does detect
non-convergence. However, all of these kernels shrink as the $\norm{\cdot}_2$
grows, which means their associated \textsc{ksd}s
become less sensitive in higher dimensions.

Suppose $\bmx, \bmy \sim N(\bm{0},
I_d)$ then  $\E [\norm{\bmx - \bmy}^2] = 2d$,
so $k(\bmx, \bmy) = \exp(-\norm{\bmx - \bmy}^2 / 2\sigma^2)$ concentrates
around $\exp(-d/\sigma^2)$. The median heuristic, $\sigma =
\text{median}({\norm{\bmx_i - \bmx_j};i <j})$, can be used to deal with this
shrinkage.
However, \citep{ramdas2015decreasing} show that even with the median heuristic,
kernel based discrepancies can converge to zero as the dimension increases
even when the distributions are different.


\section{Kernelized Complete Conditional Stein Discrepancies.}
Complete conditionals are univariate conditional distributions,
$p(x_j|\bm{x}_{-j})$, where $\bm{x}_{-j} = \{x_1, \dots x_{j - 1}, x_{j + 1}, \dots x_{d} \}$.
Complete conditional distributions are the basis for many inference
procedures including the Gibbs sampler \citep{geman1984stochastic},
and coordinate ascent variational
inference \citep{ghahramani2001propagation}.

Using complete conditionals we construct complete conditional Stein
discrepancies (\textsc{cc-sd}s) and their kernelized versions
(\textsc{kcc-sd}s). In this work we focus on the Langevin-Stein operator
\citep{barbour1990diffusion, gorham2015measuring},
defined for differentiable functions $f: \euclid^d \rightarrow \euclid^{d}$ as follows:
\begin{align*}
  (\A_{p(\bm{x})}f)(\bm{x}) = {f(\bm{x})}^{T} \nabla_{\bm{x}} \log p(\bm{x})
  + \nabla_{\bm{x}} \cdot f(\bm{x}) = \sum_{j=1}^{d} \A^{j}_{p(\bmx)} f_{j}(\bmx) \ .
\end{align*}

\paragraph{Definition.} The score function of the complete conditional,
$\nabla_{x_j} \log p(x_j \mid \bm{x}_{-j})$, is the score function of the joint,
$\nabla_{x_j} \log p(\bm{x})$. So for $f_j : \euclid^d \rightarrow \euclid$,
\begin{align*}
  \A^{j}_{p(x_{j} \mid \bm{x}_{-j})} f_{j}(\bm{x}) &= f_{j}(\bm{x}) \nabla_{x_j}
  \log p(x_{j} \mid \bm{x}_{-j}) + \nabla_{x_{j}} f_{j}(\bm{x})
  = f_{j}(\bm{x}) \nabla_{x_j}
  \log p(\bm{x}) + \nabla_{x_{j}} f_{j}(\bm{x}) \\
  &= \A_{p(\bm{x})}^{j} f_{j}(\bm{x})
\end{align*}
Using this observation, and the fact that the complete conditionals of two
distributions $p, q$ match when the distributions match,
we define the complete conditional Stein discrepancy (\textsc{cc-sd}),
$\mathcal{S}(q, \A_{p}, \mathcal{C})$ as
\begin{align}\label{eq:ccsd}
  \sum_{j=1}^{d} \E_{q(\bm{x}_{-j})} \left[ \sup_{f_j \in \mathcal{C}^j} \E_{q(x_j \mid
  \bm{x}_{-j})}[\A^{j}_{p(x_j \mid \bm{x}_{-j})}f_j(\bm{x})] \right] \ .
\end{align}
The Stein set $\mathcal{C}$ is defined
as the set of functions, $f: \euclid^d \rightarrow \euclid^d$, with
each component $f_j(\bm{x})$ satisfying
$\max \left(\norm{f_j}_{\infty}, \norm{\nabla f_j}_{\infty}, Lip(f_j) \right) \leq
1$, where $Lip(f)$ is the Lipschitz constant of $f$.
Here, the supremum is taken inside the expectation, so we have to
solve optimization problems for each dimension and each conditional.
Similar to Stein discrepancies, \textsc{cc-sd}s can be hard to compute. In
the next section, we introduce the kernelized version which has a closed form.

\subsection{Kernelized Complete Conditional Stein Discrepancies.}
We now define the Stein set, $\steinset$, for the kernelized version of
\textsc{cc-sd}, such that we get a closed form discrepancy.

We use univariate integrally symmetric positive definite (\textsc{ispd}) kernels,
$k: \euclid \times \euclid \rightarrow \euclid$,
that satisfy the following, for $g: \euclid \rightarrow \euclid$:
\begin{align}\label{eq:ispd}
  \int_{u  \in \euclid} \int_{v \in \euclid} g(u) k(u, v) g(v) du dv > 0 \ ,
\end{align}
with $\norm{g}_{2} > 0$. Let $\rkhs$ denote the reproducing kernel Hilbert space
(\textsc{rkhs}) with kernel $k$. Functions $h \in \rkhs$ satisfy the
reproducing property, $h(x_j) = \langle {h, k(x_j, \cdot)} \rangle_{\rkhs}$
for $x_j \in \euclid$.
The \textsc{rkhs} also satisfies $\Phi_{x_j}(\cdot) = k(x_j, \cdot) \in \rkhs$.

We define $\mathcal{C}_{k}$ with a univariate kernel $k$, as
consisting of functions, $f: \euclid^d
\rightarrow \euclid^d$, whose component functions $f_j:\euclid^d
\rightarrow \euclid$ satisfy $f_{j, \bm{x}_{-j}} \equiv f_j(\cdot,
\bm{x}_{-j}) \in \rkhs$ for each $\bm{x}_{-j}$. So $f_j$ with a fixed
$\bm{x}_{-j}$ is in the \textsc{rkhs} defined by $k$. This means
\begin{align}\label{eq:stein_func_cond}
  f_{j, \bm{x}_{-j}} (x_j) = \langle{f_{j, \bm{x}_{-j}}, k(x_j, \cdot)}
  \rangle_{\rkhs} \ .
\end{align}
Let $\mathcal{C}_{k}^j$ denote the
set of functions satisfying \Cref{eq:stein_func_cond} with norm bounded by
\begin{align}\label{eq:norm_condition}
  \norm{f_{j, \bm{x}_{-j}}}_{\rkhs} \leq
\norm{\E_{q(x_j \mid \bm{x}_{-j})} \left[ \A^{j}_{p(x_j \mid \bm{x}_{-j})}
    \Phi_{x_j}  \right] }_{\rkhs} \ ,
\end{align}
for all $\bm{x}_{-j} \in \euclid^{d-1}$.

We define the kernelized complete conditional Stein
discrepancy (\textsc{kcc-sd}) $\mathcal{S}(q, \A_{p}, \steinset)$ as follows,
\begin{align}\label{eq:stein_sup}
  \sum_{j=1}^{d} \E_{q(\bm{x}_{-j})} \left[ \left| \sup_{f_j \in
  \steinset^{j}} \E_{q(x_{j} \mid \bm{x}_{-j})}
  \left[ \A^{j}_{p(x_j \mid \bm{x}_{-j})} f_j(\bm{x}) \right]
  \right| \right]
\end{align}

\paragraph{KCC-SDs admit a closed form.}
In our definition of the Stein set, we can change the kernel or the kernel
parameters in each dimension, however for clarity we do not focus on that here.
Note that the Stein set depends on both distributions $p$ and $q$. We
show that the \textsc{kcc-sd} defined in \cref{eq:stein_sup} has a closed form.

\begin{theorem}[Closed form]\label{thm:closed_form}
For a kernel $k$ which is differentiable in both arguments, we define the Stein kernel for each
  $j \in \{1, \dots, d\}$ as follows:
  \begin{align}\label{eq:closed_form}
    k^{j}_{cc}(x_j, y_j ; \bm{x}_{-j}) &= \A^{j}_{p(x_j \mid \bm{x}_{-j})} \A^{j}_{p(y_j
    \mid \bm{x}_{-j})} k(x_j, y_j) \\ \nonumber
                   &= b_j(x_j, \bm{x}_{-j})
                     b_j(y_j, \bm{x}_{-j}) k(x_j, y_j)
                  + b_j(x_j, \bm{x}_{-j})
                   \nabla_{y_{j}} k(x_j, y_j) \\ \nonumber
    & \qquad + b_j(y_j, \bm{x}_{-j}) \nabla_{x_{j}} k(x_j, y_j)
      + \nabla_{x_j} \nabla_{y_j} k(x_j, y_j) \ ,
  \end{align}
  where $b_j(\bm{x})$ is equal to $\nabla_{x_j} \log p(\bm{x})$  and
  if $\E_{q(\bm{x}_{-j})} \E_{q(x_j \mid \bm{x}_{-j})}
  \E_{q(y_j \mid \bm{x}_{-j})} \left[{k_{cc}^{j}(x_j, y_j;
        \bm{x}_{-j})}^{1/2}  \right] < \infty$,
  then the \textsc{kcc-sd} can be computed in closed form as
  $\kccsd = \norm{\bm{w}}_2^2$, where the weights, $w_{j}$ are defined as $w_j^{2} = \E_{q(\bm{x}_{-j})}
    \E_{q(x_j \mid \bm{x}_{-j})}
    \E_{q(y_j \mid \bm{x}_{-j})}
    {k_{cc}^{j}(x_j, y_j; \bm{x}_{-j})}$.
\end{theorem}
The proof is in \cref{sec:closed_form}. \Cref{thm:closed_form} implies that the
functions, $f^{*}_{j}(x_j; \bm{x}_{-j})$, which achieve
the supremum in \Cref{eq:stein_sup} are
\begin{align}\label{eq:sup_function}
  f^{*}_j(x_j;\bm{x}_{-j}) &= \E_{q(y_{j} \mid \bm{x}_{-j})} \left[ \A^{j}_{p(y_j \mid
  \bm{x}_{-j})} \Phi_{x_j} \right] \\ \nonumber
  &= \E_{q(y_{j} \mid \bm{x}_{-j})} [ k(x_j, y_j) \nabla_{y_j} \log p(y_j \mid
    \bm{x}_{-j}) + \nabla_{y_j} k(x_j, y_j) ] \ ,
\end{align}
where $\nabla_{y_j} \log p(y_j \mid \bm{x}_{-j}) = \nabla_{y_j}
\log p(y_j, \bm{x}_{-j})$ and $\Phi_{x_j}(\cdot) = k(x_j, \cdot)$ is the feature
map.

We can also restrict to functions to the unit ball,
$\norm{f_{j, \bm{x}_{-j}}}_{\rkhs} \leq 1$, and still get a closed form for
the \textsc{kcc-sd}:
\begin{align}\label{eq:ugly_kccsd}
  \sum_j \E_{q(\bm{x}_{-j})} \sqrt{\E_{x_j, y_j \sim q(\cdot \mid
  \bm{x}_{-j})} {k_{cc}^{j}(x_j, y_j; \bm{x}_{-j})}} \ .
\end{align}
However, the closed form cannot be easily manipulated.

\paragraph{KCC-SDs can distinguish two distributions.}
We show that $\kccsd = 0$ if and only if $p = q$.
This proof relies on the \textsc{ispd} property of the kernel and an
equivalent form of the Stein operator when the score function
of $q$ exists. For
$f: \euclid^{d} \rightarrow \euclid^{d}$, note that
as $\E_{q(\bm{x})} \left[ \A_{q(\bm{x})} f(\bm{x}) \right] = 0$,
\begin{align*}
  \E_{q(\bm{x})} & \left[ \A_{p(\bm{x})} f (\bm{x}) \right] =
  \E_{q(\bm{x})} \left[ \A_{p(\bm{x})} f (\bm{x}) -  \A_{q(\bm{x})}
                   f(\bm{x}) \right]
  = \E_{q(\bm{x})} \left[ {f(\bm{x})}^{T} \nabla_{\bm{x}}
     \left( \log p(\bm{x})- \log q(\bm{x}) \right) \right] \ .
\end{align*}
Using this representation, we prove that if $p$ is equal to
$q$ in distribution, then \textsc{kcc-sd} is zero.

\begin{theorem}\label{thm:pisq}
Suppose $k$ is an \textsc{ispd} kernel and twice differentiable in both arguments, and
$\E_{q(\bm{x})}[\left\| \nabla_{\bm{x}} \log p(\bm{x})\right\|^2 ],
\E_{q(\bm{x})}[\left\| \nabla_{\bm{x}} \log q(\bm{x})\right\|^2 ]  < \infty$
where $p(\bm{x}),q(\bm{x}) > 0$ for all $\bm{x}\in \euclid^d$. If
$p\overset{d}{=}q$, then $ \kccsd=0$.
\end{theorem}
This property can be see by noting that when both $p$
and $q$ have score functions, their difference will be zero
inside the operator.
The proof is available in \cref{sec:detect_non_convergence}. Similarly if $p$ is not equal to $q$ in distribution, \textsc{kcc-sd}
will be able to detect that.
\begin{theorem}\label{thm:pisnotq}
  Let $k$ be integrally strictly positive definite. Suppose if
  $\kccsd < \infty$, and $\E_{q(\bm{x})}[\left\|
    \nabla_{\bm{x}} \log p(\bm{x})\right\|^2 ],
  \E_{q(\bm{x})}[\left\| \nabla_{\bm{x}} \log q(\bm{x})\right\|^2 ]  < \infty $
  with $p(\bm{x}), q(\bm{x}) > 0$,
  then if $p$ is not equal to $q$ in distribution, then $\kccsd > 0$.
\end{theorem}
The proof is in \cref{sec:detect_non_convergence}.
Combined with the previous result, this shows that \textsc{kcc-sd}s
are non-negative and zero only when the two distributions are equal.


\section{\textsc{kcc-sd} in practice}\label{sec:kccsd_practice}
Computing the optimal test function in \textsc{kcc-sd}s,
$f^{*}_{j}(x_j; \bm{x}_{-j})$, requires sampling from the complete
conditionals, $y_j \sim q(\cdot \mid \bm{x}_{-j})$. In this section, we
detail how to compute \textsc{kcc-sd} when the complete conditionals
can be sampled. We also present a sampling procedure which can be
used to compute a lower bound of \textsc{kcc-sd} when the complete conditionals
cannot be exactly sampled.

\paragraph{Exact \textsc{kcc-sd}.}
In Algorithm~\ref{alg:exact_kccsd} in \Cref{sec:closed_form} we describe how to
compute \textsc{kcc-sd}s, given a dataset $\{ \bmx^{i} \}$
and complete conditionals $q(\cdot \mid \bm{x}_{-j})$ which
can be sampled. For instance, \textsc{kcc-sd}s can
be used to assess the sample quality of samples
from a Gibbs sampler.
Here the Gibbs sampler can be used to
generate multiple auxiliary coordinates
$y^{(i, k)}_{j} \sim p(\cdot \mid \bm{x}^{(i)}_{-j})$ using the sampling
procedure for the complete conditional used in the Gibbs sampler. The auxiliary coordinate variables
can be used to compute \textsc{kcc-sd} and can be used to assess the
quality of the empirical distribution $q_n$ defined by the samples
${\{\bm{x}^{(i)}\}}_{i=1}^n$.

\paragraph{Approximate \textsc{kcc-sd}.}

Sampling from the complete conditional can be infeasible
in several scenarios.
To resolve this, we introduce approximate \textsc{kcc-sd}s, $\mathcal{S}_{\lambda}(q, \A_{p}, \steinset)$.
Suppose $g_{j}(\bm{x}) = \E_{r_{\lambda_j}(y_j \mid \bm{x}_{-j})}
[\A^{j}_{p(y_j \mid \bm{x}_{-j})} \Phi_{x_j}]$, where $r_{\lambda_{j}}$ is a
conditional distribution, then we define approximate
\textsc{kcc-sd} as
\begin{align*}
  \mathcal{S}_{\lambda}(q, \A_{p}, \steinset) = \sum_{j=1}^d
  \E_{q(\bm{x}_{-j})} \E_{q(x_{j} \mid \bm{x}_{-j})} \A^{j}_{p(x_j \mid
  \bm{x}_{-j})} g_j(\bm{x}) .
\end{align*}

Algorithm~\ref{alg:approx_kccsd_alg} in \Cref{sec:appendix_approx_kccsd} summarizes how to compute
approximate \textsc{kcc-sd}. We split the dataset ${\{\bmx\}}_{i=1}^{n}$ into a
training, validation and test set. We train a sampler on the training set
and select the model based on the lowest loss on the validation set, and then
generate samples $y_j$ from that model. \textsc{kcc-sd} is then computed on the
test set.

The reduction to probabilistic regression can make use of powerful models, such
as conditional kernel density estimation \citep{hansen2004nonparametric}
or neural network based models.
The quality of approximate \textsc{kcc-sd} depends on the performance of
the learned sampler on held-out data; this performance can be checked
on a validation set.
Formally, if the distributions ${\{r_{\lambda_{j}}\}}_{j=1}^{d}$
satisfy a $\rho$-transport inequality (Definition 3.58,
\citep{wainwright2019high}) and satisfy
$\sup_{\bm{x}_{-j}} \textsc{kl}( q(\cdot | \bmx_{-j})
\mid \mid r_{\lambda_j}) < \epsilon_j$, then we can bound the
difference between approximate \textsc{kcc-sd} and \textsc{kcc-sd}.
\begin{lemma}\label{lemma:gen_bound}
  Suppose the model class $r_{\lambda_j}$ satisfies a $\rho$-transport
  inequality and $\nabla_{\bm{x}} \log p(\bm{x})$ is Lipschitz and
  $\E_{q}[\norm{\nabla_{\bm{x}} \log p(\bm{x})}],
  \E_{r_{\lambda_j}}[\norm{\nabla_{x_j} \log p(x_j \mid \bmx_{-j})}] < \infty$, and the kernel $k$
  is bounded with $\nabla_{x_j} k(x_j, y_j)$ Lipschitz, then
  \begin{align*}
    \left| \mathcal{S}(q, \A_p, \mathcal{C}_k) - \mathcal{S}_{\lambda}(q, \A_p,
    C_{k}) \right| \leq \sum_{j=1}^{d} K_{1, j} \sqrt{2\rho^{2} \epsilon_j} +
    \sqrt{K_{2, j} \sqrt{2 \rho^{2} \epsilon_j}}
  \end{align*}
  where $\sup_{\bmx_{-j}} \textsc{kl}( q(\cdot | \bmx_{-j}) \mid \mid r_{\lambda_j}) < \epsilon_j$ and $K_{1, j}, K_{2, j}$ are positive constants.

\end{lemma}

The proof is in \Cref{sec:gen_error}. This gives us a selection criterion for
selecting models, models with a lower validation loss have approximate
\textsc{kcc-sd} values closer to \textsc{kcc-sd}.

\paragraph{Goodness of Fit testing.}
To show the efficacy of \textsc{kcc-sd} and approximate \textsc{kcc-sd} in
distinguishing distributions, we introduce a goodness-of-fit test to test
whether a given set of samples come from a target distribution. Let
the null be $H_0: p = q$, and the alternate be $H_1: p \neq q$. We do not compute the asymptotic null
distribution of the normalized test statistic, instead we use the wild-bootstrap
technique \citep{shao2010dependent,fromont2012kernels,chwialkowski2014wild,chwialkowski2016kernel}.
Define the function $h$ as
\begin{align*}
  h(\bmx^{(i)}) = \sum_{j=1}^{d} \frac{1}{m} \sum_{k=1}^{m}
                  k_{cc}(x_j^{(i)}, y^{(i, k)}_j; \bmx^{(i)}_{-j}) ,
\end{align*}
where $y^{(i, k)}_j \sim q(\cdot \mid \bmx^{(i)}_{-j})$. The test statistic
$T_n$ and the bootstrapped statistic $R_n$ are defined as
\begin{align*}
  T_n &= \frac{1}{n} \sum_{i=1}^n h(\bmx^{(i)})  \text{ and } R_{n} =
        \frac{1}{n} \sum_{i=1}^n \epsilon_{i} h(\bmx^{(i)}) ,
\end{align*}
where $\epsilon_i$ are independent Rademacher random variables and $\bmx^{(i)}$
are independently and identically distributed from $q$.
Sampling from the
complete conditional is not always computationally feasible, therefore we
propose another test with approximate \textsc{kcc-sd} as the test statistic,
which samples $y^{(i, k)}_j$ from the model $r_{\lambda_j}$.

When the null hypothesis is true, the test statistic $T_n$ converges to zero
(see \Cref{thm:pisq} for \textsc{kcc-sd} and \Cref{lemma:approx_kccsd_null} in
\Cref{sec:gft_appendix} for approximate \textsc{kcc-sd}), while $R_n$ converges to zero
under both hypotheses. In \Cref{sec:gft_appendix}, we show that $\sqrt{n}R_n$ is a good
approximation of $\sqrt{n}T_n$, so we can sample $R_n$ and approximate the quantiles
of the null distribution.

When using \textsc{kcc-sd}, under the
alternate hypothesis, $T_n$ converges to a positive constant (see \Cref{thm:pisnotq})
while $R_n$ converges to $0$. Therefore, we reject
the null hypothesis almost surely. The test can be formulated as
\begin{enumerate}
\item Compute the test statistic $T_n$.
\item Compute the estimates ${\{R_{n, l}\}}_{l=1}^{L}$.
\item Estimate the $1 - \alpha$ empirical quantile of the samples.
\item Reject the null if $T_n$ exceeds the quantile.
\end{enumerate}

When using approximate \textsc{kcc-sd}, under the null $T_n \rightarrow 0$ due
to Stein's identity (see \Cref{lemma:approx_kccsd_null} in \Cref{sec:gft_appendix}) and
the $p$-values are uniform. However, under the
alternate the asymptotic behavior of approximate \textsc{kcc-sd} depends
on the model class $r_{\lambda_j}$.
We show in the experiments that approximate
\textsc{kcc-sd} has power $1$ in comparison to baselines such
as \textsc{ksd}, \textsc{r}$\Phi$\textsc{sd} and \textsc{fssd-opt}.



\section{Experiments}
We study \textsc{kcc-sd} and approximate \textsc{kcc-sd} on comparing distributions,
selecting parameters in samplers for Bayesian neural networks, and assessing the quality
of Gibbs samplers for probabilistic matrix factorization on movie ratings.

For computing \textsc{r}$\Phi$\textsc{sd}, we use the hyperbolic secant kernel with the
median heuristic \citep{huggins2018random}. For the rest, we use
the \textsc{rbf} kernel, $k(\bmx, \bmy) = \exp(- \norm{\bmx - \bm{y}}^2 / 2\sigma^2)$.
\textsc{kcc-sd} uses $\sigma = 1$, \textsc{ksd} uses the
median heuristic, and \textsc{fssd-opt} learns the optimal
$\sigma$ parameter. For \textsc{fssd-opt} we use the code and
settings used by the authors in \citet{jitkrittum2017linear}.

To compute approximate \textsc{kcc-sd} we use a model for $r_{\lambda_j}$ based on
histograms. Suppose the samples $x_{j}$ are in an interval $I$.
Divide the interval $I$ into $m$ bins with width $\frac{1}{m}$ and
learn a neural network $f_{\theta_j}(\bm{x}_{-j})$ which predicts the
bin of $x_j$ from $\bm{x}_{-j}$. Sampling proceeds by sampling
from the categorical distribution $b_k \sim Cat(f_{\theta_j}(\bm{x}_{-j}))$,
and returning the average of the bin corresponding to $b_k$, the
sample from the categorical distribution. See \Cref{sec:appendix_experiments}
for details.
\begin{figure}[t]
  \centering
\includegraphics[scale=0.15]{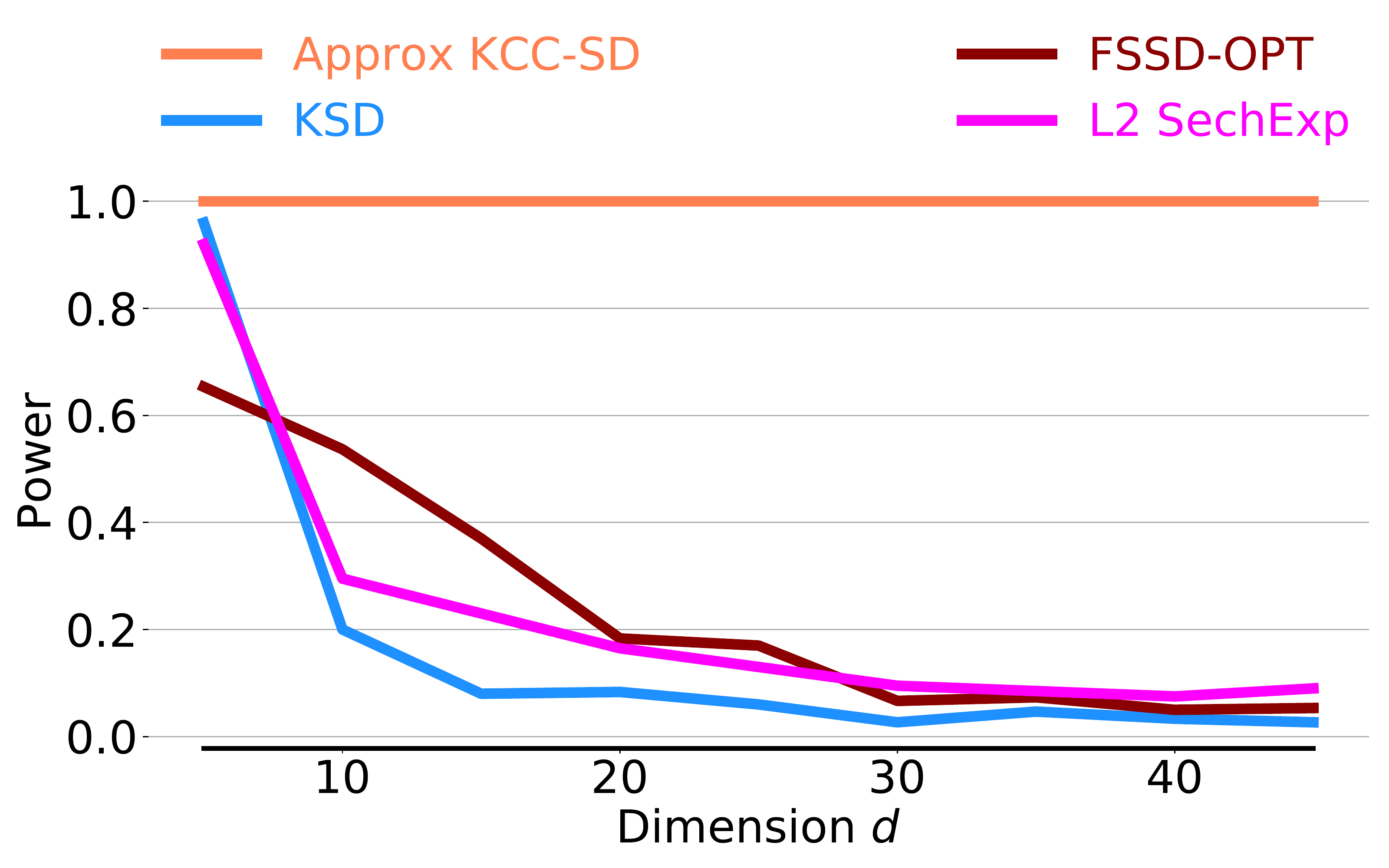}\includegraphics[scale=0.15]{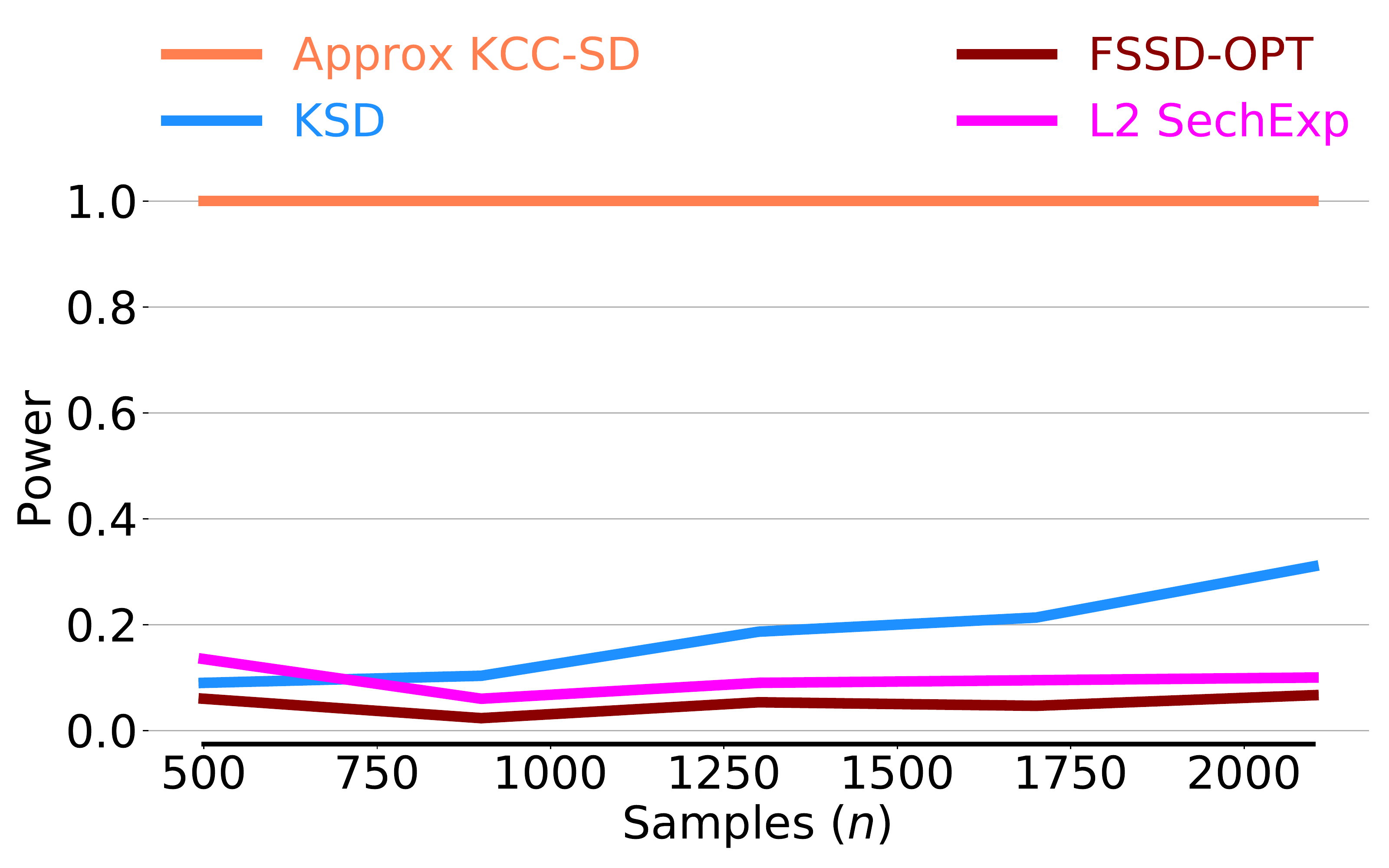}\includegraphics[scale=0.15]{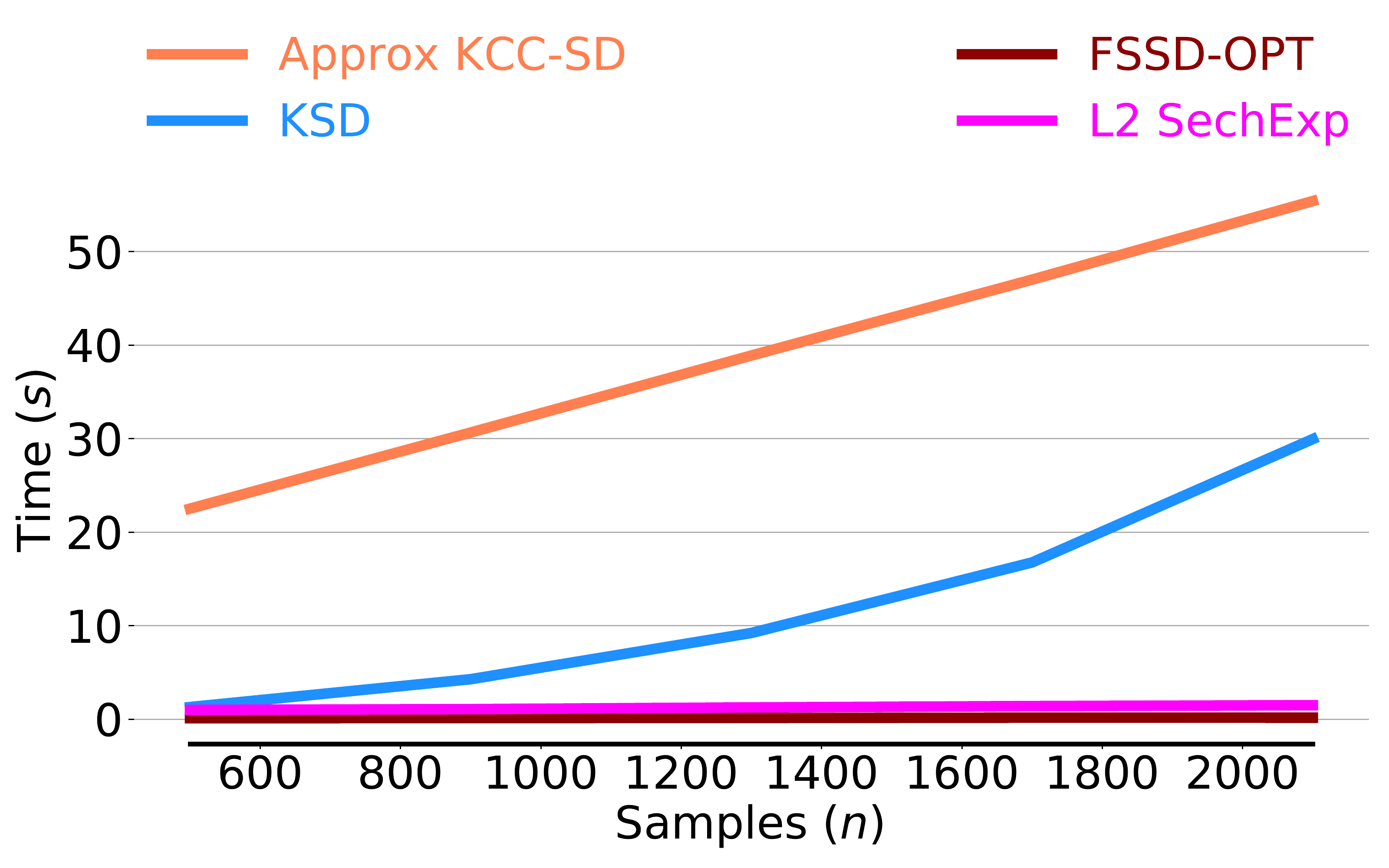}
  \caption{\textbf{\textsc{kcc-sd} has more power in high
      dimensions.} \textbf{Left:} Gaussian vs Laplace, with $n=1000$ and
    increasing dimension. Approximate \textsc{kcc-sd} has no loss in power
    compared to baseline methods. \textbf{Middle and
      Right:} Gaussian vs Laplace, with $d=30$ and increasing
    sample size. For all sample sizes studied, approximate
    \textsc{kcc-sd} has much higher power than the baseline methods, without
    requiring significantly more compute time.}\label{fig:intro}
\end{figure}

\begin{figure}[t]
  \centering
\includegraphics[scale=0.15]{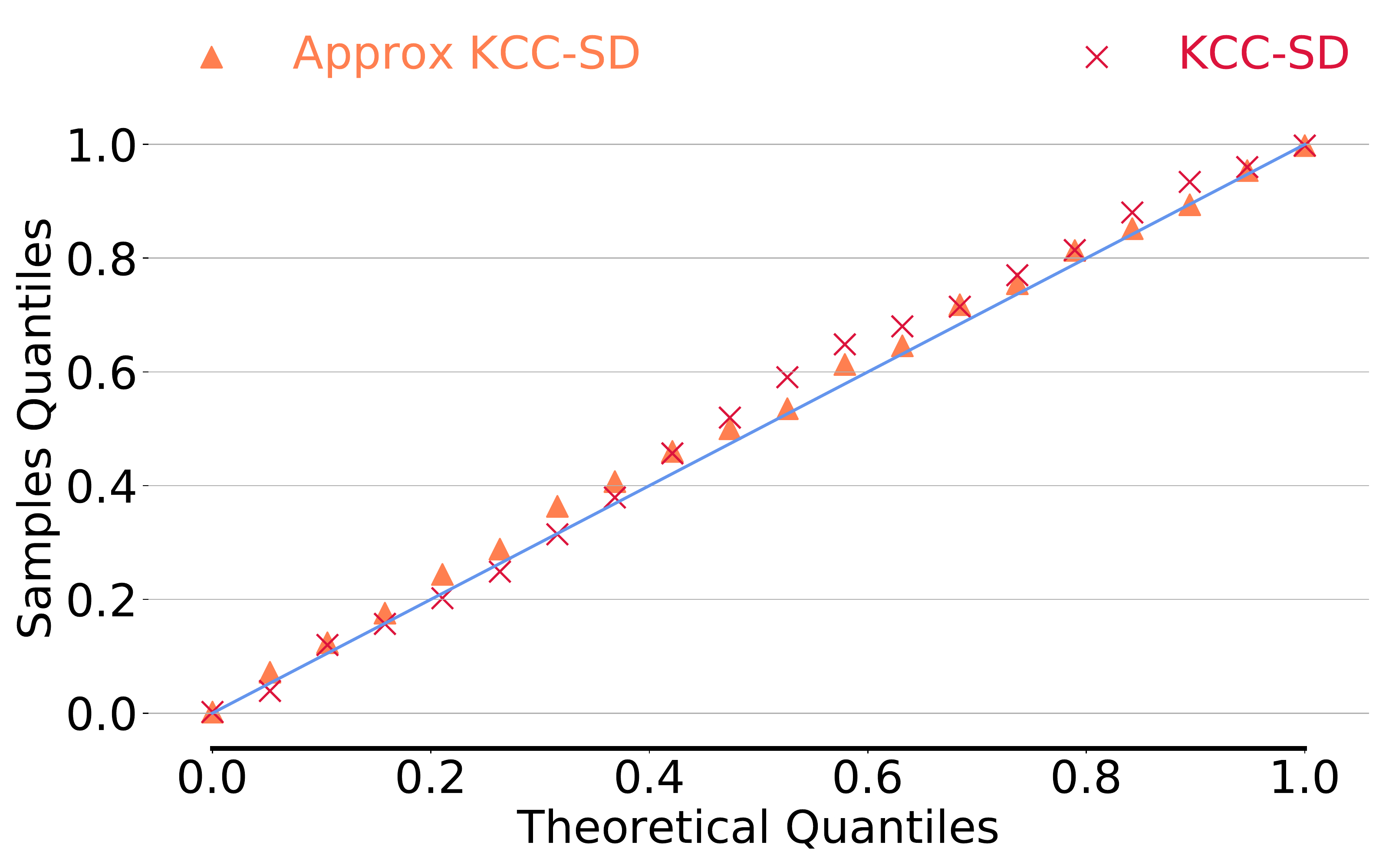}\includegraphics[scale=0.15]{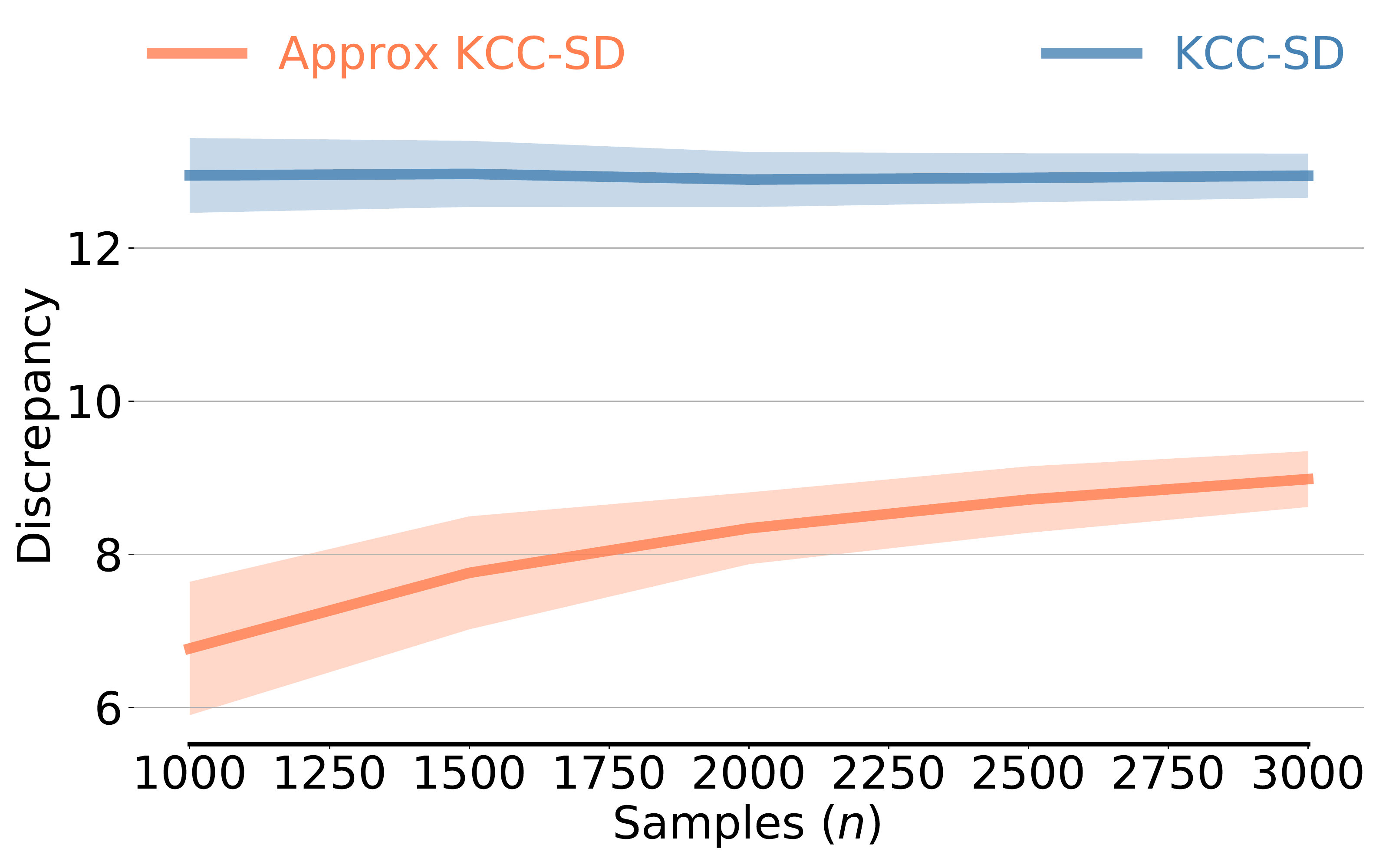}\includegraphics[scale=0.15]{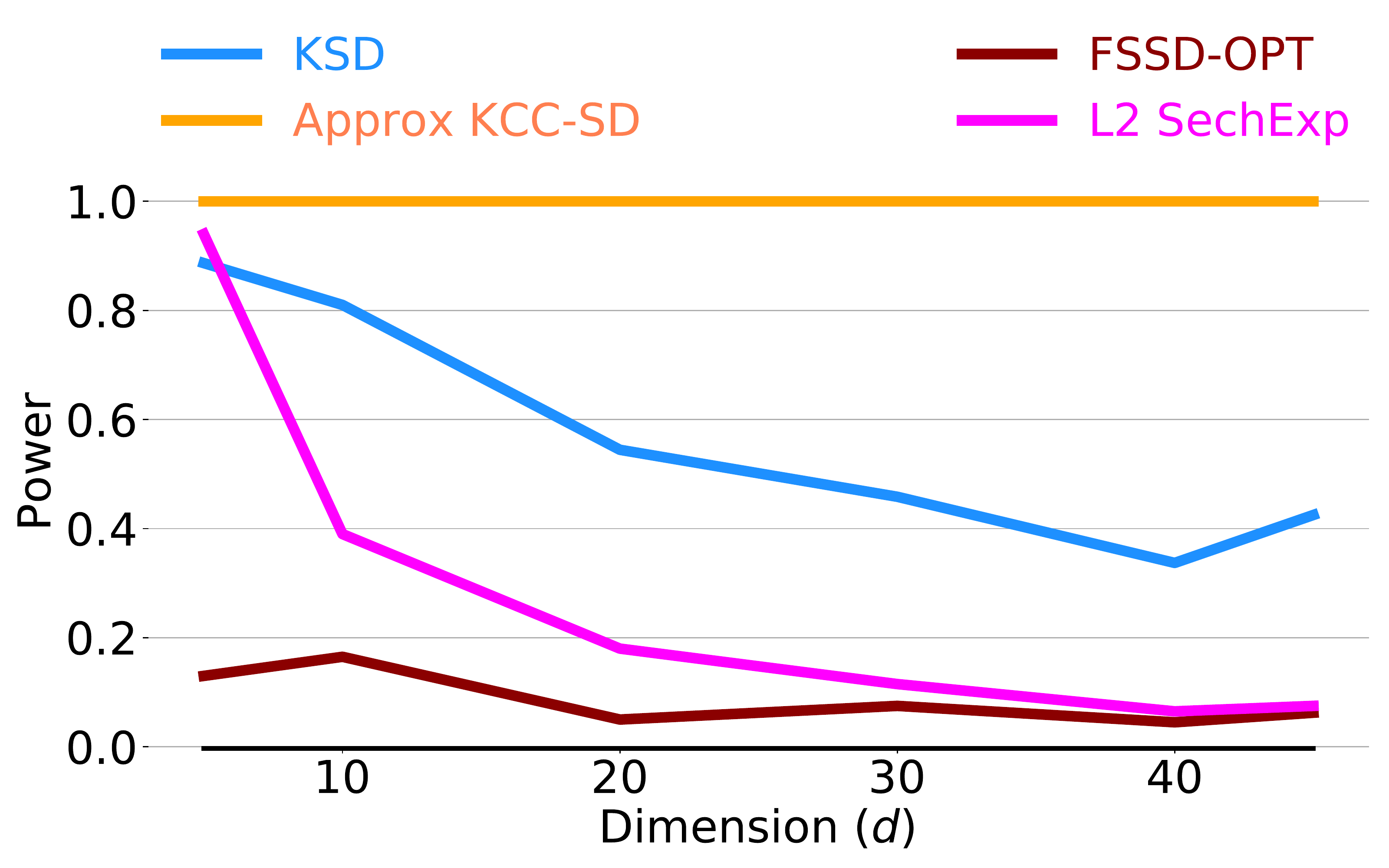}
  \caption{\textbf{Left:} Quantile-Quantile plot showing that \textsc{kcc-sd} and approximate
    \textsc{kcc-sd} have uniform $p$-values under the null, this was computed
    with $d=30$ and $n=3000$. \textbf{Middle:} Here we plot the value of
    \textsc{kcc-sd} and approximate \textsc{kcc-sd} with $p = N(0, I_d)$ and $q
    = N(0, \Sigma)$. Here the marginals match but $p \neq q$. As the number of
    samples increase, both discrepancies stay bounded away
    from zero. \textbf{Right:} Correlated Gaussian vs Correlated Gaussian with
    Laplace noise. As the dimension increases \textsc{kcc-sd} does not see a
    decrease in performance unlike the baseline methods.}\label{fig:intro2}
\end{figure}

\paragraph{Goodness-of-fit Tests.} \quad In the left panel of \Cref{fig:intro}
we compare samples from $q = \prod_{i=1}^{d} \text{Laplace}(0, 1 / \sqrt{2})$
and target density $p = N(\bm{0}, I_d)$ with increasing dimension. We generate $n=1000$ samples to
compute the test statistics, and compute the power of the test
over $300$ repetitions with a significance level $ \alpha = 0.05$. We then
observe that as the dimension increases, approximate \textsc{kcc-sd} has
power $1$ while other methods see a substantial decrease
in power as dimension increases. We show in
\Cref{sec:appendix_experiments} that similar results hold for the \textsc{imq} kernel.

In the middle panel of \Cref{fig:intro} we plot the power of the test
with $q = \prod_{i=1}^{d} \text{Laplace}(0, 1 / \sqrt{2})$ and $p = N(\bm{0}, I_d)$
with $d=30$. We then increase the number of samples used to
compute the test statistics. And in the right panel of \Cref{fig:intro} we show
the time used to compute approximate \textsc{kcc-sd}, \textsc{ksd}, \textsc{r}$\Phi$\textsc{sd} and
\textsc{fssd-opt}, the time for approximate \textsc{kcc-sd} also includes the
training time for the models. We observe that although approximate \textsc{kcc-sd} requires
more time to compute, it has more power than the baselines.

In the left panel of \Cref{fig:intro2} we compare $p = q =
N(\bm{0}, \Sigma)$ in $d=30$, with $\Sigma_{i, j} = 0.5$ for all
$i\neq j$ otherwise $\Sigma_{i, i} = 1.0$. We show that for $n=3000$ the
distribution of the $p$-values is uniform.

In the middle panel of \Cref{fig:intro2}, we
have $p = N(\bm{0}, I_d)$ and $q = N(\bm{0},
\Sigma)$ where $\Sigma_{i, j} = 0.5$ for $i \neq j$ and $\Sigma_{i, i} = 1$.
The figure shows that both \textsc{kcc-sd} and approximate \textsc{kcc-sd}
detect the differences between these
distributions.

In the right panel of \Cref{fig:intro2}, we have $p=N(\bm{0}, \Sigma)$ with
$\Sigma_{i, j} = 0.5$ and $\Sigma_{i, i} = 2$ and
samples $\bmx_i = \bm{z}_i + \bm{\epsilon}_i$, where $\bm{\epsilon}_i \sim \prod_{j=1}^d
\text{Laplace}(0, 1/\sqrt{2})$ and $\bm{z}_i \sim N(\bm{0},
\Sigma_1)$ with ${(\Sigma_{1})}_{i, j} = 0.5$ and ${(\Sigma_1)}_{i, i} = 1$,
and $\bm{z}_i$ and $\bm{\epsilon}_i$ are independent.
The samples from $q$ have the same mean and variance as $p$.
We compute $n=500$ samples and increase the
dimension. As the dimension increases, the power of the test with
approximate \textsc{kcc-sd} remains $1$, while the baseline methods
see a decline in power.

\begin{figure}[t]
  \centering
\includegraphics[scale=0.15]{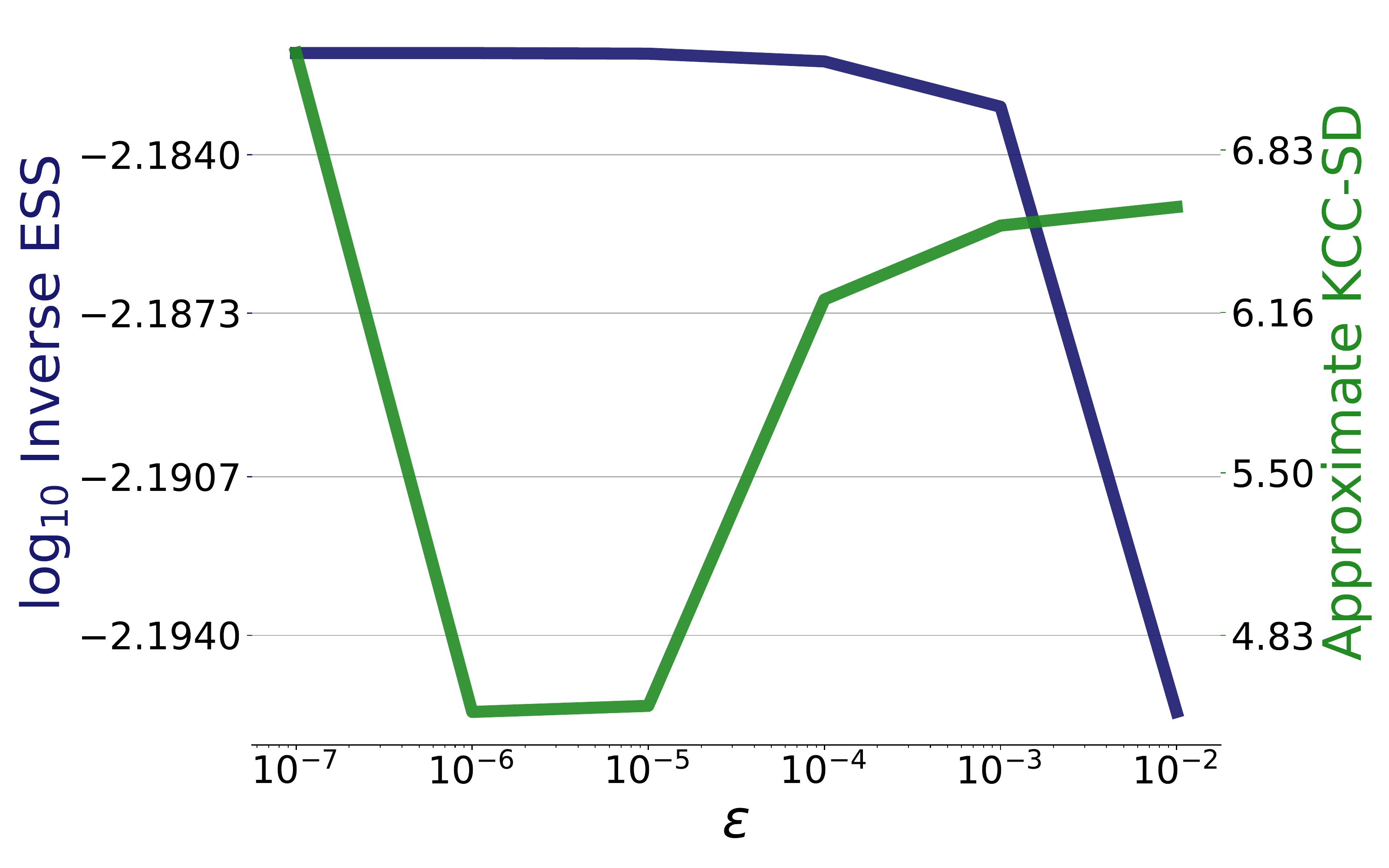}\includegraphics[scale=0.15]{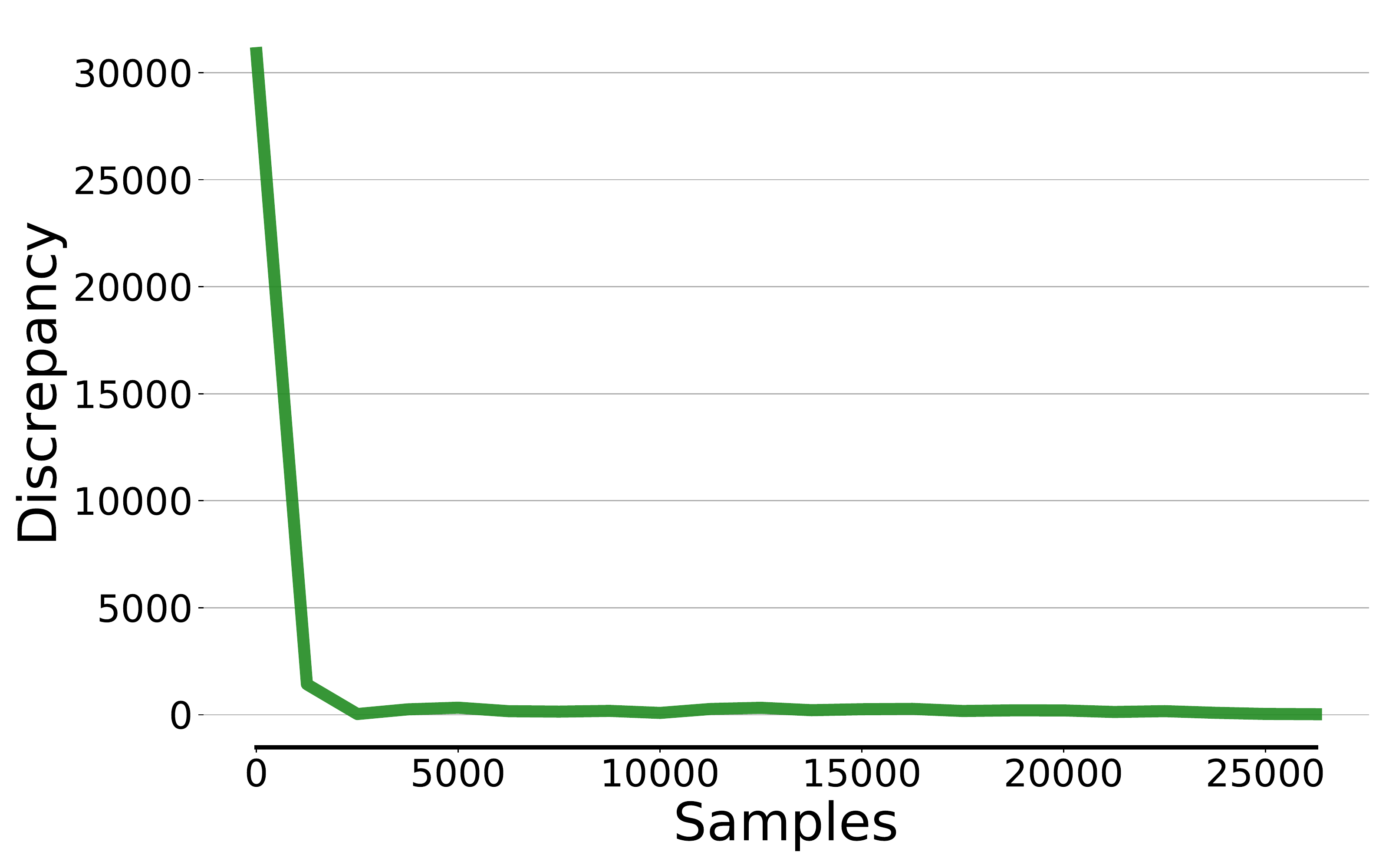}
  \caption{\textbf{Left:} Here, we plot the $\log$ inverse \textsc{ess} for
    comparison to approximate \textsc{kcc-sd} in assessing quality of samples
    from \textsc{sgld}. As we can see the inverse \textsc{ess} is
    minimized at $10^{-3}$, and \textsc{kcc-sd} is minimized at $10^{-5}$.
     \textbf{Right:} The value of block \textsc{kcc-sd}
    decreases when the number of iterations goes up in the Gibbs sampler used
    for Bayesian Probabilistic Matrix Factorization.}\label{fig:panel2}
\end{figure}

\paragraph{Selecting Biased Samplers.}

In this experiment we do posterior inference
for a three-layer neural network, with a sigmoid activation function, for
a regression task. The hidden dimensions are $40$ and $10$. We make use of
stochastic gradient Langevin dynamics (\textsc{sgld}), a
biased \textsc{mcmc} sampler \citep{welling2011bayesian}. We
used the yacht hydrodynamics dataset \citep{gerritsma1981geometry} from
the \textsc{uci} dataset repository. Since biased methods trade
sampling efficiency for asymptotic exactness,
standard \textsc{mcmc} diagnostics like effective sample size are not
applicable as they do not account for asymptotic bias.
Selecting the stepsize $\bm{\epsilon}$ is
an important task to ensure the samples are approximately
from the posterior \citep{welling2011bayesian}. For $\bm{\epsilon} \in
[10^{-8}, 10^{-3}]$ we run a chain generating 10,000
samples with a burnin phase of 50,000 samples, with
minibatch 256. We compare approximate \textsc{kcc-sd} to effective sample size.
The left panel in \Cref{fig:panel2} compares these two metrics.
While $\bm{\epsilon} = 10^{-6}$ has the lowest \textsc{kcc-sd}
value, the inverse effective sample size measure is minimized by the value
$\bm{\epsilon} = 10^{-2}$.

\paragraph{Detecting Convergence of a Gibbs Sampler for Matrix Factorization.}
We assess the convergence of a Gibbs sampler for
Bayesian probabilistic matrix factorization
\citep{salakhutdinov2008bayesian}. We focus on
a variant with two mean parameters
$\mu_V$ and $\mu_U$ for user and movie feature vectors $U_i \in
\euclid^{10}, V_j \in \euclid^{10}$ and
fixed the covariance matrix to the identity (see
\Cref{sec:appendix_experiments} for details).

In this experiment, we chose a subset of the Netflix Prize dataset, with $943$
users and $1682$ movies. We sampled the posterior
$p(\bm{\mu}_U, \bm{\mu}_V, \bm{U}, \bm{V} \mid \bm{R})$ in
blocks $\{\mu_U, \mu_V, U_1, \dots, U_N, V_1, \dots,V_M\}$ by a
Gibbs sampler. We ran the sampler for $26$K iterations with no
burnin. Since the Gibbs sampler samples blocks of variables together, using
these blocks of coordinates to compute \textsc{kcc-sd} is more efficient. In
\Cref{sec:appendix_approx_kccsd} we describe block \textsc{kcc-sd}.
We compute block \textsc{kcc-sd} by taking every $5^{th}$ sample
and show the results in the right panel of
\Cref{fig:panel2}. As the number of samples increases, block \textsc{kcc-sd}
goes down. The sample quality of the Gibbs sample increases with the number
of iterations.


\section{Discussion}\label{sec:discussion}

We developed kernelized complete conditional Stein discrepancies and
approximate \textsc{kcc-sd} and corresponding goodness-of-fit tests. We show that
these discrepancies can distinguish distributions
which have smooth and integrable score functions. We also showed empirically that
approximate \textsc{kcc-sd} provides a higher power test than those based on \textsc{ksd}. An interesting avenue of research would be relaxing the score
function requirement for $q$ and to compare the relative efficiency of the test
based on \textsc{kcc-sd} and approximate \textsc{kcc-sd} with baseline methods.

\section*{Broader Impact}
Our work focuses on comparing distributions where one is
known in functional form up to a constant. The primary application
of this method lies in probabilistic inference. Improvement in
inference could help in building models
in domains like healthcare and neuroscience especially
to propagate uncertainty about the measurements.
However, better inference could also mean better predictive models
which can have downsides like in surveillance.


\bibliographystyle{apa}
\bibliography{kccsd_paper}

\appendix
\onecolumn

\section{Closed Form}\label{sec:closed_form}

\begin{proof}
  Define the Stein operator $\A_{p(\bm{x})}$ as follows,
  \begin{align*}
     (\A_{p(\bm{x})}f)(\bm{x}) &= \sum_{j=1}^{d}(\A^{j}_{p(x_{j}
    \mid \bm{x}_{-j})}f_{j})(\bm{x}) = \sum_{j=1}^{d} f_j(\bm{x}) \nabla_{x_j}
 \log p(\bm{x}) + \nabla_{x_j} f_j(\bm{x})
  \end{align*}
  then if for all $j$, $f_{j, \bm{x}_{-j}}$ is in the \textsc{rkhs} of a
  univariate kernel, $k$, we can use the reproducing property,
  $f_{j, \bm{x}_{-j}}(x_j) = {\langle {f_{j, \bm{x}_{-j}}, k(x_j, \cdot)}
    \rangle}_{\rkhs}$
  (\cite{steinwart2008support}).
  Now, define the feature map for each kernel
  $k_j$, $\Phi_{x_j}(\cdot) = k(x_j, \cdot)$, then as
  \begin{align*}
    \partial_{x_j} f_{j, \bm{x}_{-j}}(x_j) &= \partial_{x_j} \langle
                              {f_{j, \bm{x}_{-j}}, k(x_j, \cdot)} \rangle_{\rkhs} \\
    &= \langle {f_{j, \bm{x}_{-j}}, \partial_{x_j} k(x_j, \cdot)} \rangle_{\rkhs} \\
    &= \langle {f_{j, \bm{x}_{-j}}, \partial_{x_j} \Phi_{x_j}} \rangle_{\rkhs}
  \end{align*}
  then note that we can use the reproducing property for general differential
  operators, $\A_{p(\bm{x})}^{j}$, to get
\begin{align*}
  (\A_{p(x_j \mid \bm{x}_{-j})} f_j)(\bm{x}) &= \A_{p(x_j \mid \bm{x}_{-j})}
                                               \langle {f_{j, \bm{x}_{-j}},
    k(x_j, \cdot)} \rangle_{\rkhs} \\
    &= \langle {f_{j, \bm{x}_{-j}}, \A^{j}_{p(x_j \mid \bm{x}_{-j})} \Phi_{x_j}}
      \rangle_{\rkhs}
  \end{align*}
  Then we can define the norm of $\A_{p(x_j \mid \bm{x}_{-j})} \Phi_{x_j}$, as
  follows:
  \begin{align} \nonumber
    \langle {\A_{p(x_j \mid \bm{x}_{-j})} \Phi_{x_j}, \A_{p(y_j \mid
      \bm{x}_{-j})} \Phi_{y_j}} \rangle_{\rkhs}
    &= b_j(x_j, \bm{x}_{-j}) b_j(y_j, \bm{x}_{-j}) k(x_j, y_j) + \nabla_{x_j}
       \nabla_{y_j} k(x_j,
      y_j) \\ \nonumber
    & \qquad  + b_j(x_j, \bm{x}_{-j}) \nabla_{y_j} k(x_j, y_j) +
      b_j(y_j, \bm{x}_{-j}) \nabla k(x_j, y_j) \\
    &= k^{j}_{cc}(x_j, y_j; \bm{x}_{-j})
  \end{align}
  where $b_j(u, \bm{x}_{-j}) = \nabla_{u} \log p(u | \bm{x}_{-j})$. Then we define the
  following
  \begin{align} \nonumber
    w^{2}_j &= \E_{q(x_j \mid \bm{x}_{-j})} \E_{q(y_j \mid \bm{x}_{-j})} \left[k_j^{cc}(x_j,
              y_j; \bm{x}_{-j})\right] \\ \nonumber
    &= \E_{q(x_j \mid \bm{x}_{-j})} \E_{q(y_j \mid \bm{x}_{-j})} \left[\langle {\A_{p(x_j \mid
      \bm{x}_{-j})} \Phi_{x_j}, \A_{p(y_j \mid \bm{x}_{-j})}
       \Phi_{y_j}} \rangle_{\rkhs} \right]
    \\ \label{eq:closed_form_1}
    &= \langle {\E_{q(x_j \mid \bm{x}_{-j})} \A_{p(x_j \mid \bm{x}_{-j})} \Phi_{x_j},
      \E_{q(y_j \mid \bm{x}_{-j})} \A_{p(y_j \mid
      \bm{x}_{-j})} \Phi_{y_j}} \rangle_{\rkhs} \\
    &= \norm{\E_{q(x_{j} \mid \bm{x}_{-j})} \A_{p(x_j \mid \bm{x}_{-j})}
        \Phi_{x_j}}^{2}_{\rkhs}
  \end{align}
  where $x_j, y_j \overset{i.i.d}{\sim} q(\cdot \mid \bm{x}_{-j})$
  and where we can interchange the inner product and expectation since
  $\A_{p(x_j \mid \bm{x}_{-j})} \Phi_{x_j}$ is $q$-Bochner integrable,
  (\citet{steinwart2008support}, Definition A.5.20).

  We can find the closed form for \textsc{kcc-sd}, where \textsc{kcc-sd} is
  defined as follows:
  \begin{align*}
    \kccsd &= \sum_{j=1}^{d} \E_{q(\bm{x}_{-j})} \left[ \sup_{f_{j} \in \steinset}
    \left| \E_{q(x_j \mid \bm{x}_{-j})} \left[
   \A_{p(x_j \mid \bm{x}_{-j})}^{j} f_j(\bm{x})
    \right] \right| \right] \\
  \end{align*}
  For each $j \in \{1, \dots, d\}$, and $\bm{x}_{-j}$
  \begin{align*}
    \sup_{f_j \in \steinset} \E_{q(x_j \mid \bm{x}_{-j})} \left[ \A_{p(x_j
    \mid \bm{x}_{-j})}^{j} f_j(x)  \right]
    &= \sup_{f_j: \norm{f_j} \leq w_j^{2}}
      \langle {f_j, \E_{q(x_j \mid \bm{x}_{-j})}} \left[ \A_{p(x_j \mid
      \bm{x}_{-j})} \Phi_{x_j} \right] \rangle_{\rkhs} \\
    &= \norm{\E_{q(x_{j} \mid \bm{x}_{-j})}
      \A_{p(x_j \mid \bm{x}_{-j})} \Phi_{x_j}}^{2}_{\rkhs} \\
    &= \E_{q(x_j \mid \bm{x}_{-j})} \E_{q(y_j \mid \bm{x}_{-j})}
      \left[k^j_{cc}(x_j, y_j; \bm{x}_{-j})\right]
  \end{align*}
  hence, \textsc{kcc-sd} can be written in closed form as
  \begin{align*}
    \kccsd = \sum_{j=1}^{d} \E_{q(\bm{x}_{-j})} \E_{q(x_j \mid \bm{x}_{-j})}
    \E_{q(y_j \mid \bm{x}_{-j})} \left[k^j_{cc}(x_j, y_j; \bm{x}_{-j}) \right]
  \end{align*}

\end{proof}
Here, we show that \textsc{kcc-sd}s can be expressed as an average of univariate \textsc{ksd}s. We can compute the Stein kernel for \textsc{kcc-sd} as
\begin{align*}
  k_{cc}^{j}(x_j, y_j; \bm{x}_{-j}) &=
  k(x_j, y_j) b_j(x_j, \bm{x}_{-j}) b_j(y_j, \bm{x}_{-j}) +
  \nabla_{x_j}k(x_j, y_j) b_j(y_j, \bm{x}_{-j}) \\ & \qquad
   + \nabla_{y_j}k(x_j, y_j) b_j(x_j, \bm{x}_{-j}) + \nabla_{x_j} \nabla_{y_j}
  k(x_j, y_j) , \\
  &= \left(\A_{p(x_j \mid \bm{x}_{-j})} \A_{p(y_j \mid \bm{x}_{-j})} k \right) (x_j, y_j)
\end{align*}
where $\bm{x}_{-j} \in \euclid^{d-1}$ is fixed,
$k: \euclid \times \euclid \rightarrow \euclid$, and
$b_j(x_j, \bm{x}_{-j}) = \nabla_{x_j} \log p(x_j \mid \bm{x}_{-j})$.
Using the Stein kernel defined above
we can compute \textsc{ksd} between $p(\cdot \mid
\bm{x}_{-j})$ and $q(\cdot \mid \bm{x}_{-j})$ as follows
\begin{align*}
  \mathcal{S} {\left(q(\cdot \mid \bm{x}_{-j}), \A_{p(\cdot \mid \bm{x}_{-j})},
  \mathcal{G}_k \right)}^{2} &= \E_{q(x_j \mid \bm{x}_{-j})} \E_{q(y_j \mid \bm{x}_{-j})}
  \left[
  \left(\A_{p(x_j \mid \bm{x}_{-j})} \A_{p(y_j \mid \bm{x}_{-j})} k
  \right)(x_j, y_j) \right] \\
  &= \E_{q(x_j \mid \bm{x}_{-j})} \E_{q(y_j \mid \bm{x}_{-j})}
    \left[ k_{cc}^{j}(x_j, y_j; \bm{x}_{-j}) \right] .
\end{align*}
Therefore, \textsc{kcc-sd} can also be computed as
\begin{align*}
  \kccsd &= \sum_{j=1}^d \E_{q(\bm{x}_{-j})} \E_{q(x_j \mid \bm{x}_{-j})}
    \E_{q(y_j \mid \bm{x}_{-j})} \left[k^j_{cc}(x_j, y_j; \bm{x}_{-j}) \right] \\
  &= \sum_{j=1}^d \E_{q(\bm{x}_{-j})} \left[ \mathcal{S} {\left(q(\cdot \mid
    \bm{x}_{-j}), \A_{p(\cdot \mid \bm{x}_{-j})}, \mathcal{G}_k \right)}^{2}
    \right] .
\end{align*}
In Algorithm~\ref{alg:exact_kccsd} we show how to compute \textsc{kcc-sd} exactly when
we have samples from the complete conditionals.
\begin{algorithm}
\caption{\textbf{Computing KCC-SDs with complete conditionals}}\label{alg:exact_kccsd}
\DontPrintSemicolon
\SetAlgoLined
\KwIn{Dataset ${\{\bm{x}^{(i)}\}}_{i=1}^{n}$, $d$: dimension
  of $\bm{x}$, $n_y$: number of
  $y_j$ samples and complete conditionals $q(\cdot \mid \bm{x}_{-j})$}
\KwOut{Estimated \textsc{kcc-sd} $\hat{S}_{n}(q, \A_{p}, \steinset)$}
\For{$j \in [d]$}{
\For{$i \in [n]$}{
  Sample $y^{(i, k)}_{j} \sim q(\cdot \mid \bm{x}^{(i)}_{-j})$ for $k \in [n_y]$ \;
}
Let $\hat{w}^{2}_{j} = \frac{1}{n n_y} \sum_{i=1}^{n}
\sum_{k=1}^{n_y} k_{cc}^{j}(x_j^{(i)}, y_j^{(i, k)} ; \bm{x}^{(i)}_{-j}) $ \;
}
Let $\hat{S}_{n}(q, \A_{p}, \steinset) = \sum_{j=1}^{d} \hat{w}^{2}_j$
\end{algorithm}

\section{\textsc{kcc-sd} in practice}\label{sec:appendix_approx_kccsd}
\begin{algorithm}
  \caption{\textbf{Computing approximate KCC-SDs}. Given model class
  $r_{\lambda_j}$, compute approximate \textsc{kcc-sd}.}\label{alg:approx_kccsd_alg}
\DontPrintSemicolon
\SetAlgoLined
\KwIn{Dataset $\mathcal{D} = {\{\bm{x}^{(i)}\}}_{i=1}^{n}$, $d$: dimension
  of $\bm{x}$, $n_y$: number of $y_j$ samples, and a model class
  $r_{\lambda_j}(\cdot \mid \bm{x}_{-j})$ for each complete conditional.}
\KwOut{Approximate \textsc{kcc-sd}}
Split the dataset into training, validation and test sets. \\ \;
\For{$j \in [d]$}{
  \ Train the sampler $r_{\lambda_{j}}$ on training set. \\
  \ Select the model $r_{\lambda_j}$ with lowest validation loss. \;
  \For{$i \in [n]$}{
    Sample $y^{(i, l)}_{j} \sim r_{\lambda_{j}}(\cdot \mid \bm{x}^{(i)}_{-j})$ for $l
    \in [n_y]$ \;
}
\ Let $\hat{w}^{2}_{j} = \frac{1}{n} \sum_{i=1}^{n} \frac{1}{n_y}
\sum_{l=1}^{n_y} k_{cc}(x_j^{(i)}, y_j^{(i, l)} ; \bm{x}^{(i)}_{-j})$. \;}
Let $\hat{S}_{\lambda}(q, \A_{p}, \steinset) = \sum_{j=1}^{d} \hat{w}^{2}_j$
\end{algorithm}

\paragraph{Block \textsc{kcc-sd}.}
In Gibbs sampling, when variables are sampled together,
using blocks of coordinates to compute \textsc{kcc-sd}
will be computationally more efficient than using single coordinates.
The complete conditional approach
still ensures that block \textsc{kcc-sd} distinguishes the
distributions $p$ and $q$. For instance, if $\bm{x} \in \euclid^d$, then let
$I_1, \dots, I_m$ be disjoint partitions of indices $\{1, \dots, d\}$ such that
$\cup_{j=1}^m I_j = \{1, \dots, d\}$, then we can
define block \textsc{kcc-sd} as
\begin{align*}
  \sum_{j=1}^{m} \E_{q(\bm{x}_{-I_j})} \sup_{f_{I_j}} \E_{q(\bm{x}_{I_j} \mid
  \bm{x}_{-I_j})} [\A^{j}_{p(\bm{x}_{I_j} \mid \bm{x}_{-I_j})}f_{I_j}(\bm{x})] \ ,
\end{align*}
here the the dimension of the kernel would depend on the block size, so $k_j:
\euclid^{I_j} \times \euclid^{I_j} \rightarrow \euclid$. The supremum of the
block \textsc{kcc-sd} is
\begin{align*}
 \sum_{j=1}^m  \E_{q(\bm{x}_{-I_{j}})} \E_{\bm{x}_{I_j}, \bm{y}_{I_j} \sim q(\cdot \mid
  \bm{x}_{-I_{j}})} \left[ k_{cc}^{j}(\bm{x}_{I_j}, \bm{y}_{I_j} ;
                                         \bm{x}_{-I_j}) \right] \ .
\end{align*}
Note that if we take all the coordinates as one block, block \textsc{kcc-sd}
is equivalent to \textsc{ksd}.

\section{Distinguishing Distributions}\label{sec:detect_non_convergence}
 Here, we rely on the \textsc{ispd} property of the kernel $k(x_j, y_j)$
  so that for any function $f: \euclid \rightarrow \euclid$, we obtain
  \begin{align*}
    \int_{u \in \euclid} \int_{v \in \euclid} f(u) k(u, v) f(v) du dv > 0
  \end{align*}
  for $\norm{f} > 0$.

  Note that we can write the Stein discrepancy as,
  \begin{align} \nonumber
    \E_{q(\bm{x})} \left[ \A_{p(\bm{x})}f (\bm{x}) \right] &=
     \E_{q(\bm{x})}\left[ \A_{p(\bm{x})}f(y) -
    \A_{q(\bm{x})}f(\bm{x}) \right] \\ \nonumber
    &= \E_{q(\bm{x})} \left[ {f(\bm{x})}^{T} \nabla_{\bm{x}} \log p(\bm{x})
       + \nabla_{\bm{x}} \cdot f(\bm{x})
      \right] - \E_{q(\bm{x})} \left[ {f(\bm{x})}^{T}
      \nabla_{\bm{x}} \log q(\bm{x}) + \nabla_{\bm{x}} \cdot f(\bm{x})
      \right] \\ \nonumber
    &= \E_{q(\bm{x})} \left[ {f(\bm{x})}^{T}
      \left( \nabla_{\bm{x}} \log p(\bm{x}) - \log q(\bm{x}) \right)
      \right] \\ \label{eq:stein_alternate}
    &= \E_{q(\bm{x})} \left[ {f(\bm{x})}^{T} \nabla_{\bm{x}} \log
       \frac{p(\bm{x})}{q(\bm{x})} \right] \ ,
  \end{align}
  using $\E_{q(\bm{x})} \left[ \A_{q(\bm{x})}f(\bm{x}) \right] = 0$.

  Using this representation for our test function,
  $f^{*}_j(\bm{x}) = \E_{q(y_j \mid \bm{x}_{-j})}[\A_{p(y_j | \bm{x}_{-j})}^{j}k(x_j, y_j)]$,
  where $y_j \sim q(\cdot \mid \bm{x}_{-j})$, we see that
  \begin{align} \nonumber
    f^{*}_j(\bm{x}) &= \E_{q(y_j \mid \bm{x}_{-j})}[\A_{p(y_j | \bm{x}_{-j})}^{j}k(x_j, y_j)]
    - \E_{q(y_j \mid \bm{x}_{-j})}[\A_{q(y_j | \bm{x}_{-j})}^{j}k(x_j, y_j)] \\ \nonumber
    &= \E_{q(y_j \mid \bm{x}_{-j})} \left[ k(x_j, y_j) \nabla_{y_j} \log \frac{p(y_j
      \mid \bm{x}_{-j})}{q(y_j \mid \bm{x}_{-j})} \right] \\ \label{eq:sup_alternate}
    &= \E_{q(y_j \mid \bm{x}_{-j})} \left[ k(x_j, y_j) \nabla_{y_j} \log \frac{p(y_j
      , \bm{x}_{-j})}{q(y_j , \bm{x}_{-j})} \right] \ ,
  \end{align}
  then using the fact that $\kccsd = \sum_{j=1}^{d} \E_{q(\bm{x})}[\A_{p(\bm{x})}^{j}
  f_{j}^{*}(\bm{x})]$, we obtain using \cref{eq:stein_alternate} and
  \cref{eq:sup_alternate}
  \begin{align*}
    \kccsd &= \E_{q(\bm{x})} \left[ {f^{*}(\bm{x})}^{T} \nabla_{\bm{x}} \log
    \frac{p(\bm{x})}{q(\bm{x})} \right] \\
    &= \sum_{j=1}^{d} \E_{q(\bm{x}_{-j})} \left[ \E_{q(x_j \mid \bm{x}_{-j})}
      \left[ f^{*}_{j}(\bm{x}) \nabla_{x_j}
      \log \frac{p(\bm{x})}{q(\bm{x})} \right] \right] \\
           &= \sum_{j=1}^{d} \E_{q(\bm{x}_{-j})} \left[ \E_{q(x_j \mid
             \bm{x}_{-j})} \E_{q(y_j \mid \bm{x}_{-j})}
             \left[ \nabla_{y_j} \log \frac{p(y_j, \bm{x}_{-j})}{q(y_j,
             \bm{x}_{-j})} k(x_j, y_j)
       \nabla_{x_j} \log \frac{p(\bm{x})}{q(\bm{x})} \right] \right] \ .
  \end{align*}
  Now, observe that for each $j \in \{1, \dots, d \}$, with $r(u, \bm{x}_{-j}) =
  \nabla_{u} \log \frac{p(u, \bm{x}_{-j})}{q(u, \bm{x}_{-j})}$, we define a
  function $h$ over $\bm{x}_{-j}$
  \begin{align}\nonumber
    h(\bm{x}_{-j}) &= \E_{q(x_j \mid \bm{x}_{-j})} \E_{q(y_j \mid \bm{x}_{-j})}
       \left[ \nabla_{y_j} \log \frac{p(y_j, \bm{x}_{-j})}{q(y_j, \bm{x}_{-j})} k(x_j, y_j)
       \nabla_{x_j} \log \frac{p(\bm{x})}{q(\bm{x})} \right] \\ \nonumber
    &= \E_{q(x_j \mid \bm{x}_{-j})} \E_{q(y_j \mid \bm{x}_{-j})}
       \left[ r(y_{j}, \bm{x}_{-j}) k(x_j, y_j) r(x_j, \bm{x}_{-j}) \right]\\ \nonumber
&= \int_{x_j} \int_{y_j} q(x_j \mid \bm{x}_{-j}) r(x_j, \bm{x}_{-j}) k(x_j,
    y_j) q(y_j \mid \bm{x}_{-j}) r(y_j, \bm{x}_{-j}) dx_j dy_{j} \\ \label{eq:hispd}
    &= \int_{x_j} \int_{y_j} g_{\bm{x}_{-j}}(x_j) k(x_j, y_j) g_{\bm{x}_{-j}}(y_j) dx_{j}
      dy_{j}
  \end{align}
  where $g_{\bm{x}_{-j}}(u) = q(u \mid \bm{x}_{-j})r(u, \bm{x}_{-j}) =
  q(u \mid \bm{x}_{-j}) \nabla_{u} \log
  \frac{p(u, \bm{x}_{-j})}{q(u, \bm{x}_{-j})}$.

The proofs in this section rely on the next lemma, which states that if the
complete conditionals match, then the distributions also match.

\begin{lemma}\label{lemma:cc_match}
If $p(\bm{x})$, $q(\bm{x})>0$ for all $\bm{x} \in \euclid ^d$ and $p(x_j|\bm{x}_{-j}) =
q(x_j|\bm{x}_{-j})$ for all $\bm{x}_{-j}$ and $j$, then $p(\bm{x}) = q(\bm{x})$.
\end{lemma}

\begin{proof}[Proof (\Cref{lemma:cc_match})]
We prove by induction. If dimension of $x$ is 2, then $p(x_1|x_2) = q(x_1|x_2)$ and $p(x_2|x_1) = q(x_2|x_1)$. Then we have
\begin{equation*}
\int \frac{p(x_1|x_2)}{p(x_2|x_1)} d x_1 = \int \frac{p(x_1)}{p(x_2)} d x_1= \frac{1}{p(x_2)},
\end{equation*}
and
\begin{equation*}
\int \frac{q(x_1|x_2)}{q(x_2|x_1)} d x_1 = \int \frac{q(x_1)}{q(x_2)} d x_1= \frac{1}{q(x_2)},
\end{equation*}
which implies
\begin{equation*}
\frac{1}{p(x_2)} = \int \frac{p(x_1|x_2)}{p(x_2|x_1)} d x_1 = \int \frac{q(x_1|x_2)}{q(x_2|x_1)} d x_1= \frac{1}{q(x_2)}.
\end{equation*}
 Therefore, $p(x_2) = q(x_2)$ for all $x_2$.$p(x_1,x_2) = p(x_1|x_2)p(x_2) = q(x_1|x_2)q(x_2) = q(x_1,x_2)$.

 Assume the dimension of $\bm{x}$ is $d$. Then we have
  \begin{equation*}
   \frac{p(\bm{x}_{-\{i,j\}})}{p(\bm{x}_{-i})} = \int
   \frac{p(\bm{x}_{-j})}{p(\bm{x}_{-i})} d x_i = \int \frac{p(x_i|\bm{x}_{-i})}{p(x_j|\bm{x}_{-j})} d x_i =     \int \frac{q(x_i|\bm{x}_{-i})}{q(x_j|\bm{x}_{-j})} d x_i = \int
   \frac{q(\bm{x}_{-j})}{q(\bm{x}_{-i})} d x_i =
   \frac{q(\bm{x}_{-\{i,j\}})}{q(\bm{x}_{-i})}
 \end{equation*}

for all $j$. Then $p(\bm{x}_{j}|\bm{x}_{-\{i,j\}}) = q(\bm{x}_{j}|\bm{x}_{-\{i,j\}})$ for all $j$. Since $\bm{x}_{-i}$ is a $(d-1)$ dimensional distribution, we can use the induction.  Since $p(\bm{x}_{j}|\bm{x}_{-\{i,j\}}) = q(\bm{x}_{j}|\bm{x}_{-\{i,j\}})$ for all $j$, by induction, we have $p(\bm{x}_{-i}) = q(\bm{x}_{-i})$. Therefore,
\begin{equation*}
p(\bm{x}) = p(x_i|\bm{x}_{-i})p(\bm{x}_{-i}) = 	q(x_i|\bm{x}_{-i})q(\bm{x}_{-i})=q(\bm{x}).
\end{equation*}

\end{proof}

Using \Cref{eq:stein_alternate} we can see that if $p \overset{d}{=} q$, then
$\E_{q}[\A_p f(\bm{x})] = 0$ for $f$ integrable and smooth. The Stein set for
\textsc{kcc-sd}, $\steinset$, consists of such functions. We restate
\Cref{thm:pisq} for clarity.

\begin{theorem*}
Suppose $k\in C^{2,2}(\euclid, \euclid)$ is an \textsc{ispd} kernel and
$\E_{q(\bm{x})}[\left\| \nabla_{\bm{x}} \log p(\bm{x})\right\|^2 ],
\E_{q(\bm{x})}[\left\| \nabla_{\bm{x}} \log q(\bm{x})\right\|^2 ]  < \infty$
where $p(\bm{x}),q(\bm{x}) > 0$ for all $\bm{x}\in \euclid^d$. If
$p\overset{d}{=}q$, then $ \kccsd=0$.
\end{theorem*}

\begin{proof}[\textbf{Proof (\Cref{thm:pisq})}]
  If $p\overset{d}{=}q$, then the score functions match and using \Cref{eq:stein_alternate},
  for all $f$ such that $\E_{q(\bm{x})} \|f(\bm{x})\|_2 < \infty$, then
  \begin{align*}
    \E_{q(\bm{x})} \left[ \A_{p(\bm{x})}f (\bm{x}) \right] &=
        \E_{q(\bm{x})} \left[ {f(\bm{x})}^{T} \nabla_{\bm{x}} \log
       \frac{p(\bm{x})}{q(\bm{x})} \right] \\&= 0
  \end{align*}
  Since all $f\in \steinset$ satisfy  $\E_{q(\bm{x})} \|f(\bm{x})\|_2 < \infty$,
  $\kccsd = 0$.

\end{proof}

Similarly, using \Cref{eq:stein_alternate} we can show that when $p \neq q$,
then \textsc{kcc-sd} will be strictly greater than zero. This relies on the
fact that if two measures are not equal, then on the set where they are not
equal, the complete conditionals will not match. We can exploit this property
to show that \textsc{kcc-sd} will not be zero for such distributions.
We restate \Cref{thm:pisnotq} for clarity.
\begin{theorem}
  Let $k$ be integrally strictly positive definite. Suppose if
  $\kccsd < \infty$, and $\E_{q(\bm{x})}[\left\|
    \nabla_{\bm{x}} \log p(\bm{x})\right\|^2 ],
  \E_{q(\bm{x})}[\left\| \nabla_{\bm{x}} \log q(\bm{x})\right\|^2 ]  < \infty $
  with $p(\bm{x}), q(\bm{x}) > 0$,
  then if $p$ is not equal to $q$ in distribution, then $\kccsd > 0$.
\end{theorem}

\begin{proof}[\textbf{Proof (\Cref{thm:pisnotq})}]
  Suppose $p \neq q$ in distribution, then by \Cref{lemma:cc_match} there exists a
  $j \in \{1, \dots, d\}$ and a set
  $B_{-j} \subset \euclid^{d-1}$, with $m_{d-1}(B_{-j}) > 0$ where $m_{d-1}$ is
  Lebesgue measure, such that
  for each $\bm{x}_{-j}\in B_{-j}$ there exists a set $A_{j, \bm{x}_{-j}} \subset \euclid$
  with $m_1(A_{j, \bm{x}_{-j}})>0$, where
  the complete conditional do not match. Then as the complete
  conditionals, $p(x_j \mid \bm{x}_{-j}), q(x_j \mid \bm{x}_{-j})$,
  do not match on $A_{j, \bm{x}_{-j}}$, the ratio of the score
  functions do not match, so for $\bm{x}_{-j} \in B_{-j}$ and $u \in A_{j, \bm{x}_{-j}}$,
  \begin{align*}
    g_{\bm{x}_{-j}}(u) = q(u \mid \bm{x}_{-j}) \nabla_{x_j} \log \frac{p(u,
    \bm{x}_{-j})}{q(u, \bm{x}_{-j})} \neq 0 \ .
  \end{align*}
  As $q$ has full support, for all $\bm{x}_{-j} \in B_{-j}$ we have $g_{\bm{x}_{-j}}(u) \neq 0$ on
  $A_{j, \bm{x}_{-j}}$, this implies that the $L_2$ norm of this function is
  not zero, $\norm{g_{\bm{x}_{-j}}}_2 \neq 0$. Thus, for $\bm{x}_{-j} \in
  B_{-j}$, by the \textsc{ispd} property of the kernel,
  \begin{align*}
    h(\bm{x}_{-j}) = \int_{x_j} \int_{y_j} g_{\bm{x}_{-j}}(x_j) k(x_j, y_j)
    g_{\bm{x}_{-j}}(y_j) dx_{j} dy_{j} > 0
  \end{align*}
  and since $m_{d-1}(B_{-j})$,
  $\E_{q(\bm{x}_{-j})}[h(\bm{x}_{-j})] > 0$. Thus, $\kccsd > 0$.
\end{proof}

\section{Proof of \Cref{lemma:gen_bound}: Bounding the gap in approximate \textsc{kcc-sd}}\label{sec:gen_error}

To prove \Cref{lemma:gen_bound} we make use of the following lemma
\citep{gorham2017measuring} to bound the difference between the expectation
of the Stein operator on different distributions.
\begin{lemma}\label{lemma:gen_bound_appendix}
  Suppose $\nabla_{\bm{x}} \log p(\bm{x})$ is Lipschitz and $L_2(q) \cap L_2(r)$, and $f$
  and $\nabla_{\bmx} f$ are uniformly bounded and Lipschitz, then we can show that
  \begin{align*}
    \left| \E_{q(\bmx)}[\A_{p(\bmx)}f(\bmx)] -
    \E_{r(\bmy)}[\A_{p(\bmy)}f(\bmy)] \right| \leq K_1 W_{2}(q, r) + \sqrt{K_2
    W_2(q, r)} ,
  \end{align*}

  where $K_1, K_2$ are positive constants.
\end{lemma}

\begin{proof}
  Suppose the score function, $s_p(\bm{x}) = \nabla_{\bmx} \log p (\bm{x})$, is
  Lipschitz and the function $f$ is bounded with
  a Lipschitz derivative then we can bound the approximation error as follows
  \begin{align*}
  \left| \E_{q(\bmx)} \left[\A_{p(\bmx)}f(\bmx) \right] -
    \E_{r(\bmy)} \left[\A_{p(\bmy)}f(\bmy)\right] \right|
   \leq & \left| \E_{q(\bmx)} \left[f(\bmx)^T s_p(\bmx)\right] - \E_{r(\bmy)}
    \left[f(\bmy)^T s_p(\bmy)\right] \right| \\
  & \qquad \qquad + \left| \E_{q(\bmx)} \nabla_{\bmx} \cdot f(\bmx) -
    \E_{r(\bmy)} \nabla_{\bmy} \cdot f(\bmy) \right|
\end{align*}
Now, assume that $f$ is bounded and $\nabla \log p$ is Lipschitz and so is
$\nabla_{\bmx} f$. Then, we can bound the second term above as follows
\begin{align*}
  \left| \E_{q(\bmx)} \nabla_{\bmx} \cdot f(\bmx) -
    \E_{r(\bmy)} \nabla_{\bmy} \cdot f(\bmy) \right| \leq L(\nabla f) \E [{\norm{\bmx -
  \bmy}_2}] ,
\end{align*}
where $L(h)$ is the Lipschitz constant of the function $h$ and $B(h) = \sup_{\bmx} \norm{h(\bmx)}_{2}$.

Similarly, we split the first term as follows
\begin{align*}
  \left| \E_{q(\bmx)} \left[f(\bmx)^T s_p(\bmx)\right] - \E_{r(\bmy)}
    \left[f(\bmy)^T s_p(\bmy)\right] \right|
  & \leq \left| \E \left[f(\bmx)^T \left( s_p(\bmx) -
    s_p(\bmy) \right)\right] \right|  \\  &\quad \quad + \left| \E [s_p(\bmy)^T
    (f(\bmy) - f(\bmx))] \right| .
\end{align*}
We can then bound the first term above using the fact that the function $f$ is
bounded and the score function is Lipschitz.
\begin{align*}
  \left| \E \left[f(\bmx)^T \left( s_p(\bmx) -
  s_p(\bmy) \right)\right] \right| \leq B(f) L(\nabla \log p)
  \E[\norm{\bmx - \bmy}_2]
\end{align*}
and similarly we can bound the second term by using the fact that the function
$f$ is bounded and Lipschitz and the the score function is square integrable,
\begin{align*}
  \left| \E [ s_p(\bmy)^{T}
  (f(\bmy) - f(\bmx))] \right| &\leq \E [\norm{f(\bmy) - f(\bmx)}_2
  \norm{s_p(\bmy)}_{2}] \\
  & \leq \E \left[ \min \left( 2 B(f), L(f) \norm{\bmx - \bmy}_2 \right)  \norm{s_p(\bmy)}\right]
\end{align*}
and then using the fact that $\min(a, b) \leq \sqrt{ab}$ for $a, b \geq 0$, and
applying Cauchy-Schwarz again we obtain
\begin{align*}
  \left| \E [ s_p(\bmy)^{T}
  (f(\bmy) - f(\bmx))] \right| & \leq \sqrt{2 B(f) L(f)} \E \left[ \norm{\bmx -
                                 \bmy}_2^{\frac{1}{2}} \norm{s_p(\bmy)}_2
                                 \right] \\
  & \leq \sqrt{2 B(f) L(f)} \sqrt{\E\left[ \norm{\bmx - \bmy}_2 \right] \E \left[
    \norm{s_p(\bmy)}_2^{2} \right] } .
\end{align*}
We can then bound the all the terms by
\begin{align*}
    \left| \E_{q(\bmx)}[\A_{p(\bmx)}f(\bmx)] -
    \E_{r(\bmy)}[\A_{p(\bmy)}f(\bmy)] \right| \leq K_1 \E \left[ \norm{\bmx -
  \bmy}_2 \right] + \sqrt{K_2 \E[\norm{\bmx - \bmy}_2]} ,
\end{align*}
where $K_1 = L(\nabla f) + B(f) L(\nabla \log p)$ and $K_2 = 2 B(f)L(f)
\E_{r}[\norm{s_{p}(\bmy)}_2]$ are constants. Now, by taking the infimum over
all joint distributions $P$ on $\bmx$ and $\bmy$, where the marginals match
with $q$ and $r$, we obtain
\begin{align*}
  \left| \E_{q(\bmx)}[\A_{p(\bmx)}f(\bmx)] -
    \E_{r(\bmy)}[\A_{p(\bmy)}f(\bmy)] \right| \leq K_1 W_2(q, r) + \sqrt{K_2 W_2(q, r)},
\end{align*}
where the Wasserstein distance is defined as $W_{2}(p, q) = \inf_{P, \bmx \sim
  p, \bmy \sim q}\E_{P} \left[ \norm{\bmx - \bmy}_2 \right]$.

\end{proof}

Suppose a distribution $q$ satisfies a $\rho$-transport inequality
(Definition 3.58, \citep{wainwright2019high}), then for any distribution $p$ we have the
following inequality
\begin{align}\label{eq:rho_transport}
  W_2 (q, p) \leq \sqrt{2 \rho^2 D(p, q)} ,
\end{align}
where $D$ is the \textsc{kl} divergence. We make use of the $\rho$-transport
inequality to bound the Wasserstein-$2$ in the bound proved in
\Cref{lemma:gen_bound_appendix}.

For brevity we refer to the complete conditional $q(x_j \mid \bm{x}_{-j})$ as
$q_{\mid \bm{x}_{-j}}(x_j)$. And we restate \Cref{lemma:gen_bound} below for reference.

\begin{lemma}\label{lemma:gen_bound}
  Suppose the model class $r_{\lambda_j}$ satisfies a $\rho$-transport
  inequality and $\nabla_{\bm{x}} \log p(\bm{x})$ is Lipschitz and
  $\E_{q}[\norm{\nabla_{\bm{x}} \log p(\bm{x})}],
  \E_{r_{\lambda_j}}[\norm{\nabla_{x_j} \log p(\bm{x} \mid \bmx_{-j})}] < \infty$, and the kernel $k$
  is bounded with $\nabla_{x_j} k(x_j, y_j)$ Lipschitz, then
  \begin{align*}
    \left| \mathcal{S}(q, \A_p, \mathcal{C}_k) - \mathcal{S}_{\lambda}(q, \A_p,
    C_{k}) \right| \leq \sum_{j=1}^{d} K_{1, j} \sqrt{2\rho^{2} \epsilon_j} +
    \sqrt{K_{2, j} \sqrt{2 \rho^{2} \epsilon_j}}
  \end{align*}
  where $\sup_{\bmx_{-j}} \textsc{kl}( q(\cdot | \bmx_{-j}) \mid \mid r_{\lambda_j}) < \epsilon_j$ and $K_{1, j}, K_{2, j}$ are positive constants.

\end{lemma}

\begin{proof}[Proof of \Cref{lemma:gen_bound}]
  Suppose the complete conditionals $r_{\lambda_j}$ satisfy a
  $\rho$-transport inequality \Cref{eq:rho_transport}. Then suppose $f$ is
  bounded and has a Lipschitz derivative then using \Cref{lemma:gen_bound_appendix}
  we get the following bound for each $j$,
  \begin{align*}
    \left| \E_{q(x_j \mid \bm{x}_{-j})}[\A_{p(x_j \mid \bm{x}_{-j})}f(x_j)] -
  \E_{r_{\lambda_j}(y_j \mid \bm{x}_{-j})}[\A_{p(y_j \mid \bm{x}_{-j})}f(y_j)] \right|
    &\leq  K_{1, j} W_{2}(r_{\lambda_j}, q) + \sqrt{K_{2, j} W_{2}(r_{\lambda_j}, q)} \\
    & \leq  K_{1, j} \sqrt{2\rho^2 D(q_{\mid \bm{x}_{-j}}, r_{\lambda_j})} \\
    & \quad \quad + \sqrt{K_{2, j}
    \sqrt{2 \rho^2 D(q_{\mid \bm{x}_{-j}}, r_{\lambda_j})}} .
  \end{align*}

Note \Cref{lemma:gen_bound} follows from \Cref{lemma:gen_bound_appendix} as the
function $h(y_j) = \E_{q(x_j \mid \bm{x}_{-j})}[\A_{p(x_j \mid \bm{x}_{-j})}
k(x_j, y_j)]$ satisfies the boundedness and Lipschitz assumption for
\Cref{lemma:gen_bound_appendix}. Therefore, we can show that if
$\epsilon_j = \sup_{\bmx_{-j}} \textsc{kl}(q_{\mid \bm{x}_{-j}},
r_{\lambda_j})$ then
\begin{align*}
  \left| \mathcal{S}(q, \A_p, \mathcal{C}_k) - \mathcal{S}_{\lambda}(q, \A_p,
    C_{k}) \right| &\leq \sum_{j=1}^{d} \E_{q(\bm{x}_{-j})} \left|
  \E_{r_{\lambda_j}(z_j \mid \bm{x}_{-j})} \A_p h(z_j)  - \E_{q(y_j \mid
  \bm{x}_{-j})} \A_p h(y_j) \right| \\
  & \leq \sum_{j=1}^{d} K_{1, j} \sqrt{\epsilon_j} + \sqrt{K_{2, j} \sqrt{2 \rho^{2}
    \epsilon_j}} .
\end{align*}

\end{proof}

\section{Goodness of fit Testing}\label{sec:gft_appendix}
In this section we show that
\begin{enumerate}
\item When the null hypothesis is true, the bootstrapped statistics can be
  used to approximate quantile of the null distribution, so
  \begin{align*}
    \sup_{\beta} \left| \prob \left[ \sqrt{n}T_n > \beta
    \right] - \prob \left[\sqrt{n}R_n > \beta \mid
  {\{\bmx^{(i)}\}}_{i \leq n} \right]  \right| \rightarrow 0
  \end{align*}
   as $n \rightarrow \infty$.
  This holds when the test statistic is computed using either \textsc{kcc-sd} or
  approximate \textsc{kcc-sd}.
\item When the alternate hypothesis is true, the test statistic computed using
  \textsc{kcc-sd} converges to a positive constant almost surely
  (\Cref{thm:pisnotq}), that is $\prob[T_n > 0] \rightarrow 1$. And
  this leads to an almost sure rejection of the null asymptotically.
\item When the alternate hypothesis holds, the asymptotic behaviour of the test with
  approximate \textsc{kcc-sd} depends on the model $r_{\lambda_j}$.
\end{enumerate}

The goodness-of-fit test using approximate \textsc{kcc-sd} makes use of the
fact that approximate \textsc{kcc-sd} converges to zero as the number of
samples increases, this can be seen immediately using Stein's identity. Stein's
identity states that for bounded functions $f$ with a bounded
derivative (Proposition 1, \citet{gorham2015measuring}), which vanish at infinity,
\begin{align*}
  \E_{p(\bmx)} \left[ \A_{p(\bmx)} f(\bmx) \right] = 0 .
\end{align*}
Using Stein's identity we show that when $p = q$, approximate \textsc{kcc-sd}, $S_{\lambda}(q, \A_p,
\mathcal{C}_k)$, is zero.

\begin{lemma}\label{lemma:approx_kccsd_null}
Suppose $k$ is bounded and twice differentiable in both arguments with bounded
derivatives and both $k(x_j, y_j)$ and $\nabla_{x_j}k(x_j, y_j), \nabla_{y_j}k(x_j, y_j)$ vanish at
infinity, and the score function $\nabla_{y_j} \log p(y_j
\mid \bmx_{-j}) \in L_2(r_{\lambda_j})$ for $q(\bmx_{-j})$ almost
surely and $r_{\lambda_j}$ is a density and $\E_{q(\bmx)}[\norm{\nabla_{\bmx} \log p(\bmx)}^2_2],
\E_{q(\bmx)}[\norm{\nabla_{\bmx} \log q(\bmx)}^2_2] < \infty$. Then
if $p = q$, we have the following
\begin{align*}
  \mathcal{S}_{\lambda}(q, \A_{p}, \steinset) = \sum_{j=1}^d
  \E_{q(\bm{x}_{-j})} \E_{q(x_{j} \mid \bm{x}_{-j})} \A^{j}_{p(x_j \mid
  \bm{x}_{-j})} g_j(\bm{x}) = 0 ,
\end{align*}
where $g_{j}(\bm{x}) = \E_{r_{\lambda_j}(y_j \mid \bm{x}_{-j})}
[\A^{j}_{p(y_j \mid \bm{x}_{-j})} k(x_j, y_j)]$.
\end{lemma}

\begin{proof}[Proof of \Cref{lemma:approx_kccsd_null}]
We show that approximate
\textsc{kcc-sd} is zero when $p = q$ by using Stein's identity, which
states that for bounded functions $f$ with a bounded derivative and which
vanish at infinity, we have the following
\begin{align*}
  \E_{p(\bmx)} \left[ \A_{p(\bmx)} f(\bmx) \right] = 0
\end{align*}

We show that the function $g_{j}(x_j; \bmx_{-j})
= \E_{r_{\lambda_j}(y_j \mid \bm{x}_{-j})}
[\A^{j}_{p(y_j \mid \bm{x}_{-j})} k(x_j, y_j)]$ is bounded, with a bounded derivative
and vanishes at infinity in $x_j$ with a fixed $\bmx_{-j}$. Now, using
Cauchy-Schwarz we have
\begin{align}\label{eq:g_j_ineq}
  \left| g_j(x_j; \bmx_{-j}) \right| \leq&
  \left| \E_{r_{\lambda_j}(y_j \mid \bm{x}_{-j})}
  k(x_j, y_j) \nabla_{y_j} \log p(y_j, \bmx_{-j}) \right| +
  \E_{r_{\lambda_j}(y_j \mid \bm{x}_{-j})} \left| \nabla_{y_j} k(x_j, y_j)
  \right| \\ \nonumber
  & \leq \sqrt{\E_{r_{\lambda_j}(y_j \mid \bm{x}_{-j})} [k(x_j, y_j)^2]
    \E_{r_{\lambda_j}(y_j \mid \bm{x}_{-j})} [\nabla_{y_j} \log p(y_j,
    \bmx_{-j})^{2}]} \\ \nonumber
    &  \quad \quad + \E_{r_{\lambda_j}(y_j \mid \bm{x}_{-j})} \left| \nabla_{y_j} k(x_j, y_j)
    \right| .
\end{align}
As the kernel $k$ and its derivative are bounded and both vanish at infinity and
$\nabla_{y_j} \log p(y_j \mid \bmx_{-j}) \in L_2(r_{\lambda_j})$ for
$q(\bmx_{-j})$ almost surely, we have that the function $g_j$ is
bounded and as $x_j \rightarrow
\infty$ the function $g_j$ converges to zero.

We also show that the
function $g_j$ has a bounded derivative with respect to $x_j$. Using the
inequality from \Cref{eq:g_j_ineq} we obtain
\begin{align*}
  \left| \nabla_{x_j} g_j(x_j; \bmx_{-j}) \right| &\leq  \left| \E_{r_{\lambda_j}(y_j \mid \bm{x}_{-j})}
  \nabla_{x_j} k(x_j, y_j) \nabla_{y_j} \log p(y_j, \bmx_{-j}) \right| +
  \E_{r_{\lambda_j}(y_j \mid \bm{x}_{-j})} \left| \nabla_{y_j} \nabla_{x_j} k(x_j, y_j)
  \right| \\
  & \leq \sqrt{\E_{r_{\lambda_j}(y_j \mid \bm{x}_{-j})} [ \nabla_{x_j} k(x_j, y_j)^2]
    \E_{r_{\lambda_j}(y_j \mid \bm{x}_{-j})} [\nabla_{y_j} \log p(y_j,
    \bmx_{-j})^{2}]} \\ \nonumber & \quad \quad + \E_{r_{\lambda_j}(y_j \mid \bm{x}_{-j})}
    \left|  \nabla_{x_j} \nabla_{y_j} k(x_j, y_j) \right| .
\end{align*}
Therefore, as the function is bounded and vanishes at infinity with a bounded derivative,
using Stein's identity (Proposition 1, \citet{gorham2015measuring}) for the
univariate complete conditionals, we can show that for $q(\bmx_{-j})$ almost
surely the following holds
\begin{align*}
  \E_{q(x_j \mid \bmx_{-j})} \A^{j}_{p(x_j \mid \bmx_{-j})}g_j(x_j; \bmx_{-j}) = 0 .
\end{align*}
as $p=q$.

This implies that $\mathcal{S}_{\lambda}(q, \A_{p}, \steinset) =
\sum_{j=1}^{d}\E_{q(\bmx_{-j})} \E_{q(x_j \mid \bmx_{-j})} \A^{j}_{p(x_j \mid
  \bmx_{-j})} g_j(x_j; \bmx_{-j})= 0$.

\end{proof}

Now, we show that under the null, the bootstrapped statistics $\sqrt{n}R_n$ can
be used to estimate the quantiles of the null distribution. Define the test statistic $T_n$ as follows
\begin{align*}
  T_n &= \frac{1}{n} \sum_{i=1}^n h(\bmx^{(i)}) \\
  h(\bmx^{(i)}) &= \sum_{j=1}^{d} \frac{1}{m} \sum_{k=1}^{m} k_{cc}(x_j^{(i)}, y^{(i,
                  k)}_j; \bmx_{-j}^{(i)}) ,
\end{align*}
where $y^{(i, k)}_j \sim q(\cdot \mid \bmx_{-j}^{(i)})$ for \textsc{kcc-sd} and
$y^{(i, k)}_j \sim r_{\lambda_j}(\cdot \mid \bmx_{-j}^{(i)})$ for approximate
\textsc{kcc-sd}. Under the null we
have that $T_n \rightarrow 0$ (\Cref{thm:pisq} for \textsc{kcc-sd} and
\Cref{lemma:approx_kccsd_null} for approximate \textsc{kcc-sd}) almost surely.
Then assuming that $\bmx^{(i)} \overset{i.i.d}{\sim} q$ we have
\begin{align*}
  \sqrt{n} T_n = \frac{1}{\sqrt{n}} \sum_{i=1}^n h(\bmx^{(i)}) \Rightarrow
  N(0, \sigma^{2}_{H_0}) .
\end{align*}
as $\E[h(\bmx^{(i)})] = 0$ and $\E \left[h{(\bmx^{(i)})}^2 \right] < \infty$.

In this work, we do not compute the variance and therefore we use the wild
bootstrap procedure
\citep{shao2010dependent,fromont2012kernels,chwialkowski2014wild,
  chwialkowski2016kernel}. We then define the bootstrapped statistic $R_n$ as
\begin{align*}
  R_n &= \frac{1}{n} \sum_{i = 1}^{n} \epsilon^{i} h(\bmx^{(i)}) ,
\end{align*}
where $\epsilon_i$ are independent Rademacher random variables.
Then note that under the null and under the alternate $R_n \rightarrow 0$
almost surely. We also observe that under the null
\begin{align*}
  \sqrt{n} R_n = \frac{1}{\sqrt{n}} \sum_{i=1}^{n} \epsilon^{i} h(\bmx^{(i)}) \Rightarrow N(0,
  \sigma^{2}_{H_0}) .
\end{align*}
Note, the mean
and variance of the normalized bootstrapped statistics match that of the
normalized test statistic $\sqrt{n} T_{n}$,
\begin{align*}
  \E[\epsilon_i h(\bmx)] = 0, \text{and } \E[\epsilon_i
    h{(\bmx)}^2] = \E [h{(\bmx^{(i)})}^2] .
\end{align*}
Therefore, under the null we have
$ \sup_{\beta} \left| \prob \left[ \sqrt{n}T_n > \beta \right] - \prob
    \left[\sqrt{n}R_n > \beta \mid
  {\{\bmx^{(i)}\}}_{i=1}^{n} \right]  \right| \rightarrow 0$.

Under the alternative hypothesis, we note that by \Cref{thm:pisnotq}, $T_{n} \rightarrow
C > 0$, when using \textsc{kcc-sd}. While, $R_n \rightarrow 0$ almost surely,
therefore as $\prob[T_n > 0] \rightarrow 1$ we reject the null almost surely.

When using approximate \textsc{kcc-sd} as a test statistic, the
probability of rejection asymptotically is controlled by the quality of the
model $r_{\lambda_j}$ as can be seen using \Cref{lemma:gen_bound},
\begin{align*}
  \left| \mathcal{S}_{\lambda}(q, \A_{p}, \steinset) - \mathcal{S}(q, \A_{p},
  \steinset) \right| \leq \sum_{j=1}^{d} K_{1, j} \sqrt{\epsilon_j} + \sqrt{K_{2, j} \sqrt{2 \rho^{2}
    \epsilon_j}} .
\end{align*}
where $\epsilon_j = \sup_{\bmx_{-j}} \textsc{kl}(q_{\mid \bm{x}_{-j}},
r_{\lambda_j})$.

\section{Experiments}\label{sec:appendix_experiments}
\begin{figure}[t]
  \centering
\includegraphics[scale=0.15]{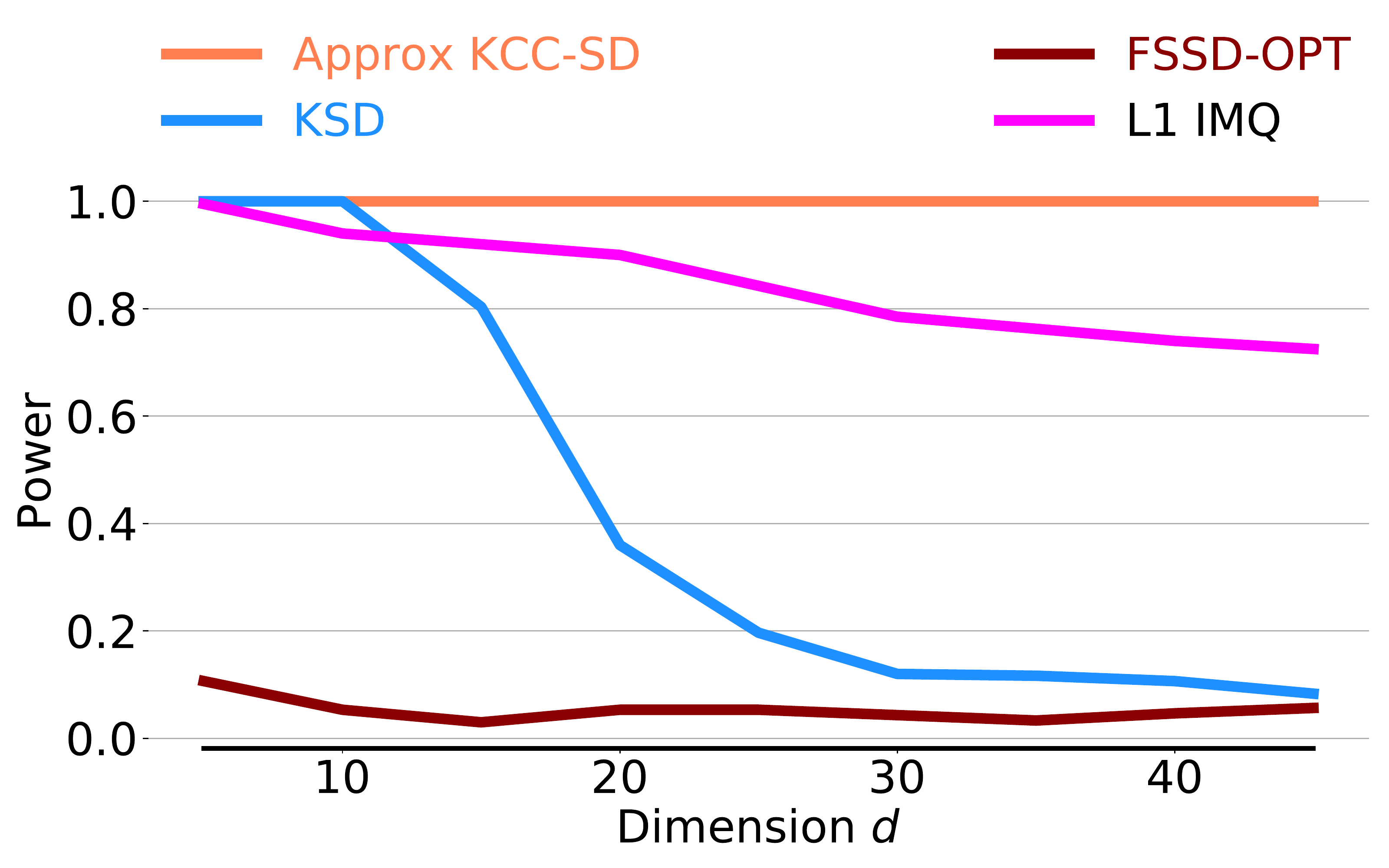}\includegraphics[scale=0.15]{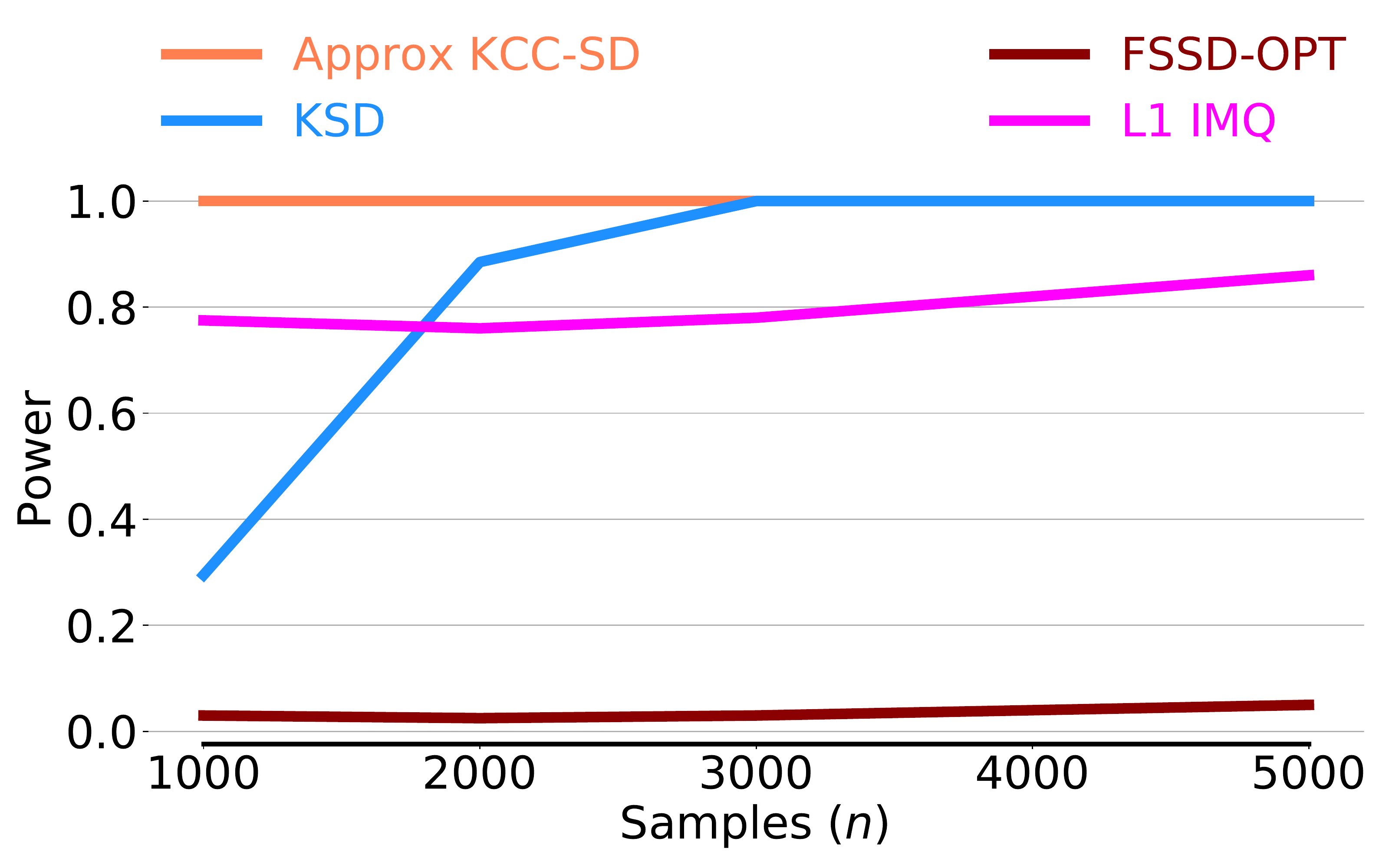}\includegraphics[scale=0.15]{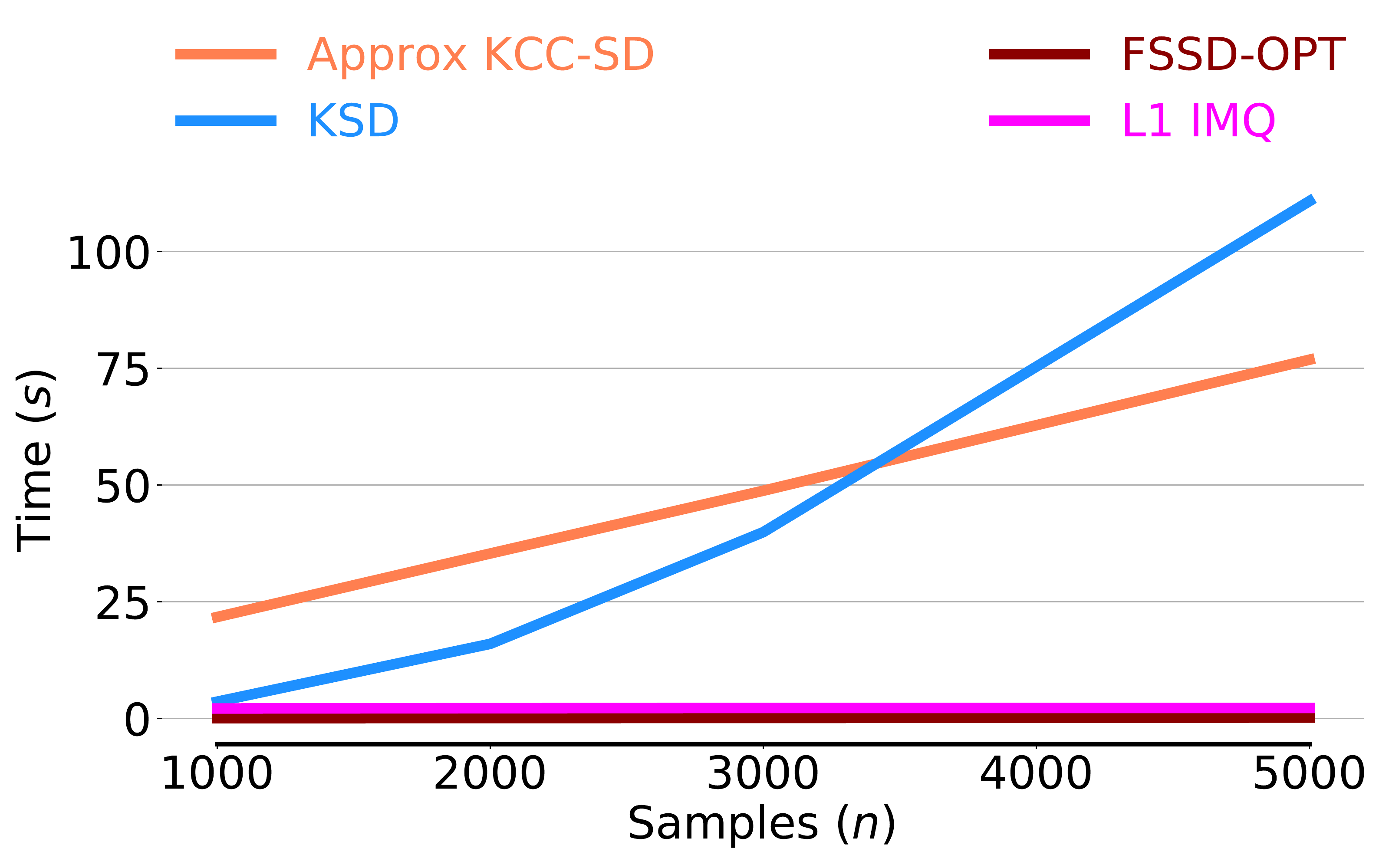}
  \caption{\textbf{\textsc{kcc-sd} has more power than baseline methods with
  the \textsc{imq} kernel.} \textbf{Left:} We compute $n=1000$ samples from
 $q = \prod_{i=1}^d \text{Laplace}(0, 1/\sqrt{2})$ with target density
 $p=N(0, I_d)$, and plot the power with increasing dimension. \textbf{Middle
   and Right:} Here we have samples from the same $p$ and $q$ distributions as
  before, but for fixed $d=30$ we increase the number of samples $n$ to show the
  number of samples required and computation time for baseline methods to achieve
  similar power as \textsc{kcc-sd}.}\label{fig:gft_imq_appendix}
\end{figure}

For the histogram-based sampler we use in our experiments, we use a two-layer
neural network with $15$-dimensional hidden-layer with a sigmoid
activation function. We train the model with gradient descent
for $500$ epochs. We select the model with the lowest validation loss.

For \textsc{fssd-opt} we use $20\%$ of the samples for training and in
approximate \textsc{kcc-sd} we use $20\%$ for training and $10\%$ for validation.

The $p$-value is computed as the proportion of the bootstrapped statistics, $R_n$,
greater than the test statistic, $T_n$. And the power is computed as
the proportion of $p$-values less than the significance level, in other words
the power is the rejection rate of the null when the alternate hypothesis
is true.

\paragraph{Choice of Kernel and Goodness-of-Fit tests.}
All the experiments done with approximate \textsc{kcc-sd}, \textsc{ksd} and
\textsc{fssd-opt} were done using the \textsc{rbf} kernel. The \textsc{rbf} kernel, a $C_0$ kernel (Definition 4.1,
\citet{carmeli2010vector}), suffices in defining consistent goodness-of-fit tests
when comparing independent samples from a distribution
$q$ (Theorem 2.2, \citet{chwialkowski2016kernel}).

\citet{gorham2017measuring} construct
a sequence of empirical distributions, $q_n$, which does not converge to
any distribution, a non-tight sequence. However, they prove that when comparing
sequences $q_n$ with $p=N(\bm{0}, I_d)$ \textsc{ksd} with the \textsc{rbf} kernel still converges to zero.
For this purpose they show that if \textsc{ksd} is computed with the
\textsc{imq} kernel then \textsc{ksd} can enforce
uniform tightness (Theorem 6, \citet{gorham2017measuring}).

However, as we have independent
samples from a distribution $q$, this situation does not
arise and we can use the \textsc{rbf} kernel.
In \Cref{fig:gft_imq_appendix} we repeat the experiments from \Cref{fig:intro}
with \textsc{imq} kernel. In the left panel, we compare the power of the test
using \textsc{kcc-sd}, \textsc{r}$\Phi$\textsc{sd}, \textsc{ksd} and
\textsc{fssd-opt} with $q = \prod_{i=1}^{n}
\text{Laplace}(0, 1/\sqrt{2})$ and $p = N(\bm{0}, I_d)$. We compute $n=1000$
samples and then increase the dimension.

In the center and right panel of \Cref{fig:gft_imq_appendix}, we compare the
same distribution as above in $d=30$ with an increasing number of samples. We
observe that \textsc{kcc-sd} requires less samples than the baseline methods to
have power $1$, and for the baseline methods to have the same amount of power
requires a similar amount of computation.

In the left panel of \Cref{fig:panel_appendix}, we have $p=N(\bm{0}, \Sigma)$ with
$\Sigma_{i, j} = 0.5$ and $\Sigma_{i, i} = 2$ and
samples $\bmx_i = \bm{z}_i + \bm{\epsilon}_i$, where $\bm{\epsilon}_i \sim \prod_{j=1}^d
\text{Laplace}(0, 1/\sqrt{2})$ and $\bm{z}_i \sim N(\bm{0},
\Sigma_1)$ with ${(\Sigma_{1})}_{i, j} = 0.5$ and ${(\Sigma_1)}_{i, i} = 1$,
and $\bm{z}_i$ and $\bm{\epsilon}_i$ are independent.
The samples from $q$ have the same mean and variance as $p$.
We compute $n=500$ samples and increase the
dimension. As the dimension increases, the power of the
\textsc{kcc-sd} test with the \textsc{imq} kernels remains $1$, while
the baseline methods with the \textsc{imq} kernels see a decline in power.

\begin{figure}[t]
  \centering
\includegraphics[scale=0.15]{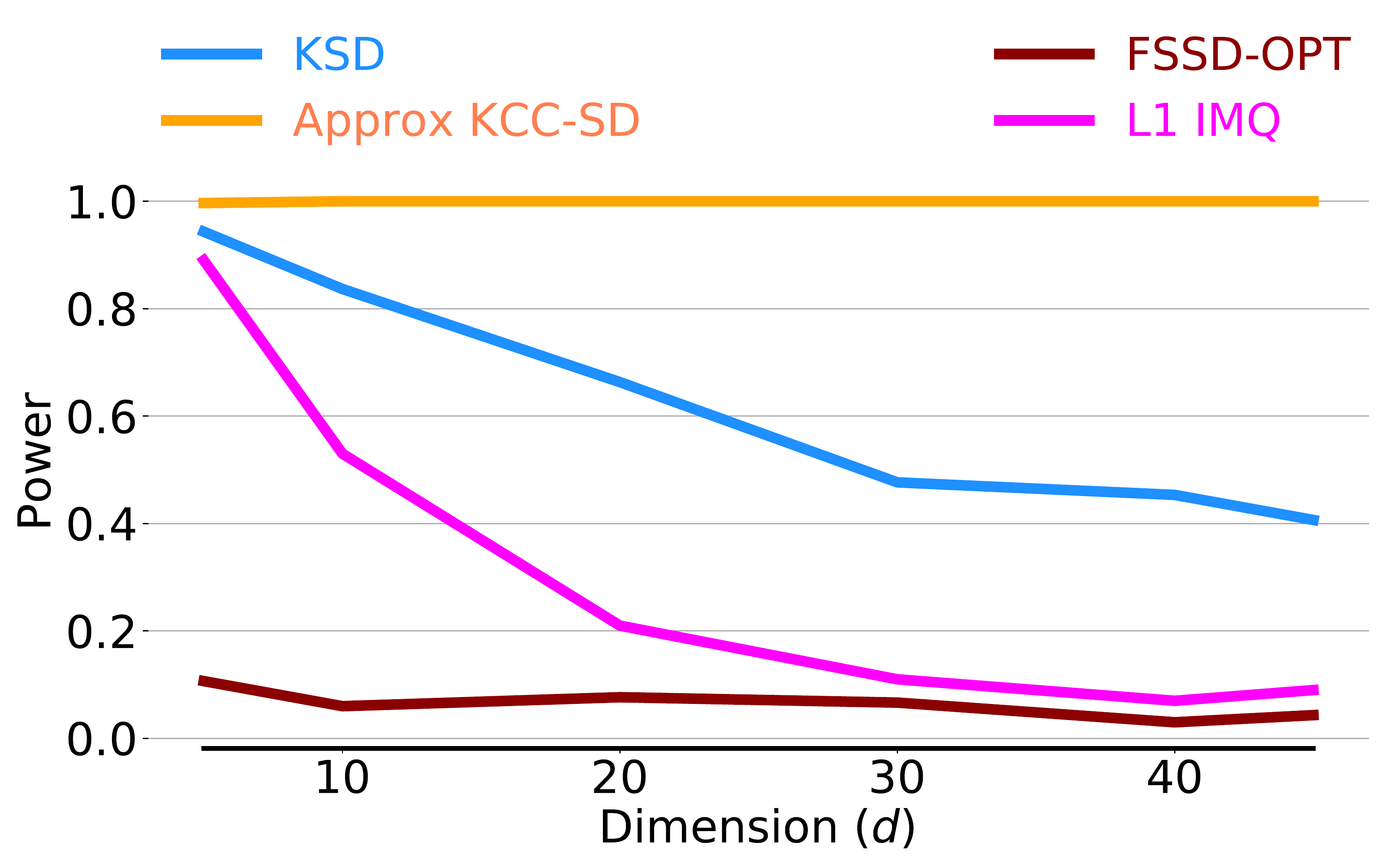}\includegraphics[scale=0.15]{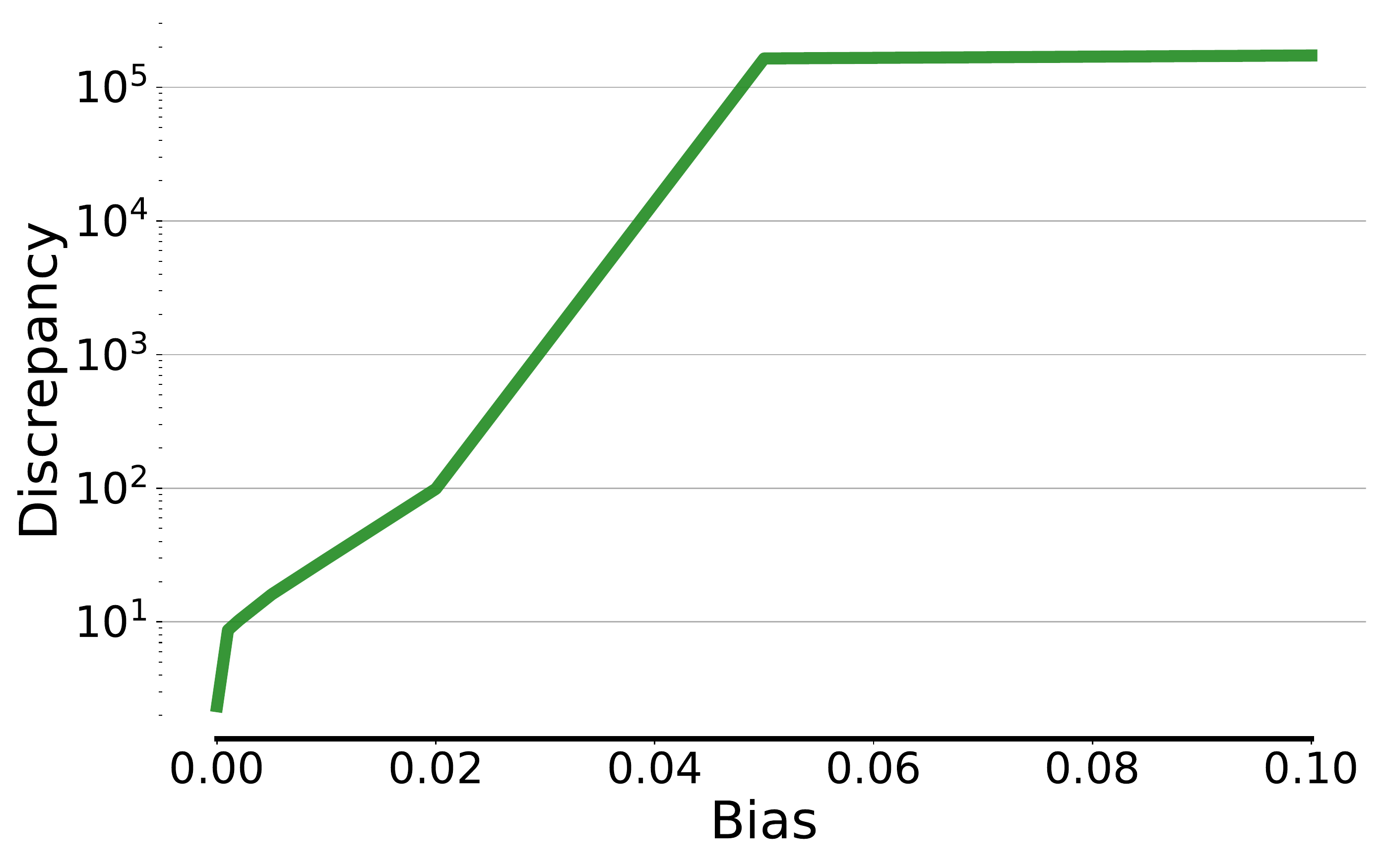}
  \caption{\textbf{Left:} Correlated Gaussian vs Correlated Gaussian with
    Laplace noise. As the dimension increases \textsc{kcc-sd} does not see a
    decrease in performance unlike the baseline methods. We use the
    \textsc{imq} kernel here. \textbf{Right:} As we add
  larger bias terms to the acceptance probability in the inner Metropolis
  sampler, samples from Metropolis-within-Gibbs sampler give larger
  \textsc{kcc-sd}. }\label{fig:panel_appendix}
\end{figure}

\paragraph{Detecting Convergence of a Gibbs Sampler for Matrix Factorization.}
Here we provide details of the probabilistic model considered in the
experiments section for assessing the convergence of a Gibbs sampler for
Bayesian probabilistic matrix factorization
\citep{salakhutdinov2008bayesian}. We focus on
a variant with two mean parameters
$\mu_V$ and $\mu_U$ for user and movie feature vectors $U_i \in
\euclid^{10}, V_j \in \euclid^{10}$ and
fixed the covariance matrix to the identity.
\begin{align*}
  p(\bm{U}|\bm{\mu}_U) &= \prod_{i=1}^N N(U_i|\bm{\mu}_U, I),
                         \quad p(\bm{\mu}_U) = N(0,I) \\
  p(\bm{V}|\bm{\mu}_V) &= \prod_{j=1}^M N(V_j|\bm{\mu}_V, I),
                         \quad p(\bm{\mu}_V) = N(0,I) \\
  p(\bm{R} \mid \bm{U}, \bm{V}) &= \prod_{i=1}^{N} \prod_{j=1}^M {\left[ N(R_{ij}
                                  \mid U_{i}^{T} V_{j}, I) \right]}^{I_{ij}}
\end{align*}
where $U_i, V_j$ have normal priors and $I_{ij}$ is the indicator variable
that is one if user $i$ rated movie $j$ and $0$ otherwise (see
\Cref{sec:appendix_experiments}).

\paragraph{Selecting Biased Samplers.}
We use a simple bimodal Gaussian mixture model to demonstrate the power of
\textsc{kcc-sd} in distinguishing biased samplers,
\begin{align*}
  x_i \sim \frac{1}{2} N(\theta_1, 2) + \frac{1}{2}(\theta_2,2) \ ,
\end{align*}
where $\theta_1, \theta_2$ have standard normal priors. We draw 100 samples
of $x_{i}$ from the model with $(\theta_1,\theta_2)=(1,-1)$. We choose
Metropolis-within-Gibbs to sample from the posterior over $\bm{\theta}$.
This sampler uses a Metropolis sampler to sample each complete conditional
inside the Gibbs sampler. We also use the Metropolis step to generate auxiliary
variables used to calculate \textsc{kcc-sd}. Denote $q(\bm{\theta})$ to be
the target distribution. The inner Metropolis step accepts the candidate
$\bm{\theta}_{new}$ with probability $\min \left(1,
  q(\bm{\theta}_{new})/q(\bm{\theta}_{old})\right)$. Then we add a bias
term to the acceptance probability, $\min \left(1,
  q(\bm{\theta}_{new})/q(\bm{\theta}_{old})+bias\right)$, thus the sampler
is not unbiased anymore. We run for 60,000 iterations in total
and drop the first 50,000 for burn-in. We show
\textsc{kcc-sd}s versus size of the bias terms in the right panel of
\Cref{fig:panel_appendix}. \textsc{kcc-sd} increases with the size of the bias.


\end{document}